\pdfoutput=1
\documentclass[11pt]{article}

\usepackage[letterpaper,margin=1in]{geometry}
\usepackage{lipsum}

\newcommand\blfootnote[1]{%
  \begingroup
  \renewcommand\thefootnote{}\footnote{#1}%
  \addtocounter{footnote}{-1}%
  \endgroup
}
\usepackage{amsthm}
\usepackage{thm-restate}

\usepackage[utf8]{inputenc} %
\usepackage[T1]{fontenc}    %

\usepackage{url}            %
\usepackage{booktabs}       %
\usepackage{nicefrac}       %
\usepackage{microtype}      %
\usepackage{enumitem}
\usepackage{dsfont}
\usepackage{multicol}
\usepackage{xcolor}
\usepackage{mdframed}

\usepackage{amsfonts,amsmath,amssymb,epsfig,color,float,graphicx,verbatim, enumitem}

\usepackage{algorithmicx}
\usepackage[noend]{algpseudocode}
\usepackage{algorithm}

\usepackage{mathtools}
\usepackage{bbm}

\usepackage{thm-restate}

\definecolor{green}{rgb}{0.0, 0.5, 0.0}
\usepackage[colorlinks,citecolor=blue,linkcolor=magenta,bookmarks=true,hypertexnames=false]{hyperref}
\usepackage[nameinlink,capitalize]{cleveref}

\crefformat{equation}{(#2#1#3)}
\crefname{lemma}{lemma}{lemmata}
\crefname{claim}{claim}{claims}
\crefname{theorem}{theorem}{theorems}
\crefname{proposition}{proposition}{propositions}
\crefname{corollary}{corollary}{corollaries}
\crefname{claim}{claim}{claims}
\crefname{remark}{remark}{remarks}
\crefname{definition}{definition}{definitions}
\crefname{fact}{fact}{facts}
\crefname{question}{question}{questions}
\crefname{condition}{condition}{conditions}
\crefname{algorithm}{algorithm}{algorithms}
\crefname{assumption}{assumption}{assumptions}
\crefname{notation}{notation}{notation}
\crefname{cond}{Condition}{Conditions}

  {%
   \par\noindent{\bfseries\upshape Proof Sketch\ }%
  }%
  {\qed}

\newtheorem{theorem}{Theorem}[section]
\newtheorem{lemma}[theorem]{Lemma}

\newtheorem{claim}[theorem]{Claim}

\newtheorem{definition}[theorem]{Definition}

\newtheorem{fact}[theorem]{Fact}
\newtheorem{question}[theorem]{Question}
\theoremstyle{definition}

\newtheorem{remark}[theorem]{Remark}

\newcommand{\eps}{\epsilon}

\newcommand{\Ind}{\mathds{1}}
\newcommand{\1}{\Ind}

\renewcommand{\Pr}{\operatorname*{\mathbb{P}}}
\newcommand{\Var}{\operatorname*{\mathrm{Var}}}

\newcommand{\E}{\operatorname*{\mathbb{E}}}

\newcommand{\poly}{\operatorname*{\mathrm{poly}}}
\newcommand{\polylog}{\operatorname*{\mathrm{polylog}}}

\renewcommand{\vec}[1]{\boldsymbol{\mathbf{#1}}}

\newcommand{\argmax}{\operatorname*{\mathrm{arg\,max}}}

\renewcommand{\d}{\mathrm{d}}

\newcommand{\trace}{\operatorname{tr}}

\def\R{\mathbb R}

\def\N{\mathbb N}

\def\Z{\mathbb Z}

\def\argmax{\mathrm{argmax}}

\newcommand{\cC}{\mathcal{C}}
\newcommand{\cD}{\mathcal{D}}

\newcommand{\cN}{\mathcal{N}}

\newcommand{\cS}{\mathcal{S}}

\newcommand{\cV}{\mathcal{V}}

\newcommand{\norm}[1]{\lVert#1\rVert}

\newcommand{\hide}[1]{}

\newcommand{\op}{\textnormal{op}}
\newcommand{\fr}{\textnormal{F}}

\def\d{\mathrm{d}}

\newcommand{\tr}{\mathrm{tr}}
\let\vec\mathbf

\def\colorful{1}

\ifnum\colorful=1

\else

\fi

\allowdisplaybreaks

\title{Robust Sparse Estimation for Gaussians with Optimal Error\\ under Huber Contamination}

\author{
Ilias Diakonikolas\thanks{Supported by NSF Medium Award CCF-2107079, NSF Award CCF-1652862 (CAREER), a Sloan Research
Fellowship, and a DARPA Learning with Less Labels (LwLL) grant.
}\\
University of Wisconsin-Madison\\
{\tt ilias@cs.wisc.edu}\\
\and
Daniel M. Kane\thanks{Supported by NSF Medium Award CCF-2107547, NSF Award CCF-1553288 (CAREER), and a Sloan Research
Fellowship.}\\
University of California, San Diego\\
{\tt dakane@cs.ucsd.edu}
\and
Sushrut Karmalkar\thanks{Supported by NSF under Grant 2127309 to the Computing Research Association for the CIFellows 2021 Project.}\\
University of Wisconsin-Madison\\
{\tt skarmalkar@wisc.edu}\\
\and
Ankit Pensia\\
IBM Research\\
{\tt ankit@ibm.com}\\
\and
Thanasis Pittas\thanks{Supported by NSF Medium Award CCF-2107079 and NSF Award DMS-2023239 (TRIPODS).}\\
University of Wisconsin-Madison\\
{\tt pittas@wisc.edu}\\
}

\begin{document}

\maketitle

\begin{abstract}%
We study Gaussian sparse estimation tasks in Huber's contamination model with a focus 
on mean estimation, PCA, and linear regression. For each of these tasks, we give 
the first sample and computationally efficient robust estimators with optimal error 
guarantees, within constant factors. All prior efficient algorithms for these 
tasks incur 
quantitatively
suboptimal error. Concretely, for Gaussian robust 
$k$-sparse mean estimation on $\R^d$ with corruption rate $\eps>0$, our 
algorithm 
has sample complexity $(k^2/\eps^2)\polylog(d/\eps)$, runs in sample polynomial time, 
and approximates the target mean within $\ell_2$-error $O(\eps)$. Previous efficient 
algorithms inherently incur error $\Omega(\eps \sqrt{\log(1/\eps)})$. At the technical level, we develop a novel multidimensional filtering method 
in the sparse regime that may find other applications.
\end{abstract}
\blfootnote{Authors are in alphabetical order.}

\setcounter{page}{0}

\thispagestyle{empty}

\newpage

\section{Introduction}
\looseness=-1
Robust statistics focuses on developing estimators resilient to a constant fraction of outliers in the sample data \cite{HubRon09,diakonikolas2023algorithmic}.
A data set may have been contaminated by outliers originating from a variety of sources: measurement error, equipment malfunction, data mismanagement, etc.
The pivotal question of developing robust estimators in statistics was first posed in the 1960s by Tukey and Huber \cite{Huber64,Tuk60}.
The standard model for handling outliers, originally formalized in \cite{Huber64}, is defined below.
\begin{definition}[Huber Contamination Model]\label{def:huber}
    Given $0 < \eps < 1/2$ and a distribution family $\cD$,
    the algorithm specifies $n \in \N$ and observes $n$ i.i.d.\ samples
    from a distribution $P = (1-\eps) G + \eps B$, where $G \in \cD$ 
    and $B$ is arbitrary. 
    We say that $G$ is the distribution of inliers, 
    $B$ the distribution of outliers, and
    $P$ is the $\eps$-corrupted version of $G$. {A set of samples generated in this fashion is called an $\eps$-corrupted set of samples from $P$.}
\end{definition}
Estimating the parameters of a Gaussian distribution --- the prototypical family of distributions in statistics --- in the Huber contamination model is a foundational problem in robust statistics~\cite{HubRon09}.
Huber in his foundational work \cite{Huber64} settled the question of {robust} {\em univariate} Gaussian mean estimation, i.e., $\mathcal{D} = \{\cN(\mu,1): \mu \in \R\}$.
Since then, a large body of work has developed sample-efficient robust estimators for various tasks~\cite{HubRon09}, that 
were unfortunately computationally inefficient for high-dimensional tasks.  
Only in the past decade have the first computationally efficient robust estimators been introduced; see \cite{diakonikolas2023algorithmic} for a book on the topic. 

Despite this remarkable progress, perhaps surprisingly, our understanding of this fundamental problem of Gaussian estimation under Huber contamination remains incomplete for \emph{structured} settings.

In many high-dimensional settings, additional structural information about the data can dramatically decrease the 
sample complexity of estimation. 
Our focus here is on the structure of {\em sparsity}. 
In the context of mean estimation, this corresponds to
the regime that at most $k$ out of the $d$ coordinates 
of the target mean $\mu$ are non-zero; our focus is on the practically relevant regime of  $k\ll d$.
Sparsity has been crucial in improving statistical performance in a myriad of applications~\cite{Hastie15}.
We begin with the problem of robust sparse mean estimation. 

\begin{definition}[Robust Sparse Mean Estimation]
\label{def:robust-sparse-mean-estimation}
Given %
$\eps\in(0,1/2)$ , $k \in \Z_+$ 
and $\eps$-corrupted samples from $\cN(\mu, \vec I)$ on $\R^d$
under Huber contamination  for an unknown $k$-sparse\footnote{We say $x$ is $k$-sparse if it has at most $k$ non-zero coordinates.} mean $\mu$, compute an estimate $\widehat{\mu}$ such that $\|\widehat{\mu} - \mu\|_2$ is small.
\end{definition}

For the above problem, it is known that the information-theoretically optimal error is $\Theta(\eps)$. 
Moreover, the sample complexity of this task is known to be $\poly(k \log d,1/\eps)$ (upper and lower bound); this should be contrasted with the unstructured setting (i.e., dense) whose  sample complexity is $\poly(d/\eps)$.
Hence, we call algorithms with sample complexity $\poly(k,\log d,1/\eps)$ sample-efficient.
However, until recently, all known 
sample-efficient
algorithms had running time %
$d^{\Omega(k)}$ (essentially amounting to brute-force search for the hidden support).
The first sample and computationally efficient
algorithm for robust sparse mean estimation was given in \cite{BDLS17}, 
but it incurred an error of $\Omega(\eps \sqrt{\log(1/\eps)})$.\footnote{We remark that their algorithm is robust to strong contamination model, which is stronger than Huber contamination model; see \Cref{sec:related-work} for a thorough discussion.}
On the other hand, \cite{DKKLMS18-soda} gave a computationally-efficient algorithm with error $O(\eps)$ for the \emph{dense} setting 
with sample complexity polynomial in $d$, thus sample-inefficient. 
This leads to the question:
\begin{question}
    Is there a sample and computationally efficient robust sparse mean estimator with $O(\eps)$ error?
\end{question}
Beyond mean estimation, other important sparse estimation tasks include principal component analysis (PCA) and linear regression, defined below:
\begin{definition}[Robust Sparse PCA]
\label{def:robust-sparse-pca}
Given  $\eps \in (0,1/2) $, spike strength $\rho >  0$, and a set of $\eps$-corrupted samples from $\mathcal{N}(0, \vec I  +   \rho vv^\top)$ for an unknown $k$-sparse unit vector $v {\in} \R^d $, compute an estimate $\widehat{v}$ 
such that $\|\widehat{v}\widehat{v}^\top  -   vv^\top\|_{\fr}$ is small.
\end{definition}
Robust PCA (in the dense setting) has been studied since \cite{XuCM13}, alas with suboptimal error.
\begin{definition}[Robust Sparse Linear Regression]
\label{def:robust-sparse-linear-regression}
For $\beta \in \R^d$ and standard deviation $\sigma > 0$, we define $P_{\beta,\sigma}$ to be  the  joint distribution over $(X,y)$ where $X\sim \cN(0,\vec I)$ and $y \sim \cN(x^\top \beta,\sigma^2)$.    
Given  $\eps \in (0,1/2) $, $\sigma > 0$, and a set of $\eps$-corrupted samples from $P_{\beta,\sigma}$ for an unknown $k$-sparse $\beta \in \R^d $, compute an estimate $\widehat{\beta}$ 
such that $\|\widehat{\beta} - \beta\|_2$ is small.
\end{definition}
Taking $\rho = \sigma  = 1$ for convenience, the optimal errors for both robust sparse PCA and linear regression are still $\Theta(\eps)$. Similarly to mean estimation, existing sample and computationally efficient 
estimators for these problems incur error $\Omega(\eps \sqrt{\log(1/\eps)})$~\cite{BDLS17}.
Focusing on linear regression, the recent work of \cite{diakonikolas2023near} gave a computationally efficient estimator \Cref{def:robust-sparse-linear-regression} achieving $O(\eps)$ error; 
but since they do not incorporate sparsity, their algorithm inherently requires $\Omega(d)$ samples. 
For robust PCA, the computational landscape is even less understood:  even 
with $\poly(d)$ samples, 
no polynomial-time estimator is known that achieves $O(\eps)$ error.
{Thus, we are led to the following question}: 
\begin{question}
\label{ques:pca-regression}
    Are there sample and computationally efficient estimators for robust sparse PCA and robust sparse linear regression that achieve $O(\eps)$ error?
\end{question}

We answer both of these questions in the affirmative.

\subsection{Our Results}
In what follows, we let $\eps_0 \in (0,1/2)$ be a sufficiently small positive constant. We start with the mean estimation result: 
\begin{restatable}[Robust Sparse Mean Estimation]{theorem}{MAINTHMMEAN}\label{thm:main}
    For any $\eps \in (0,\eps_0)$, let $T$ be an $\eps$-corrupted set of $n$ samples (in the Huber contamination model) from $\cN(\mu,\vec I)$ for an unknown $k$-sparse mean $\mu {\in} \R^d$.
    There exists an algorithm that, given corruption rate $\eps \in (0,\eps_0)$, failure probability $\delta \in (0,1)$, sparsity parameter $k \in \N$ and a dataset with $n \geq \frac{k^2\log d + \log(1/\delta)}{\eps^2} \polylog(\frac{1}{\eps})$  samples,
    computes an estimate $\widehat{\mu} \in \R^d$ such that $\|\widehat{\mu} - \mu \|_{2} = O(\eps)$ 
    with probability at least $1-\delta$.   Moreover, the algorithm runs in $\poly(n d)$-time.
\end{restatable}
The error of $O(\eps)$ is information-theoretically optimal up to constants\footnote{Observe that \Cref{thm:main} cannot be extended to identity covariance sub-Gaussian distributions, as the information-theoretic error for this class is $\Theta(\eps \sqrt{\log(1/\eps)})$.}. Importantly, \Cref{thm:main} is the first sample and computationally efficient algorithm achieving this optimal error guarantee.
Moreover, the $k^2$ dependence in the sample complexity 
is optimal 
within the class of computationally efficient algorithms~\cite{DKS17-sq,brennan2020reducibility}.
For robust sparse PCA, we show:
\begin{restatable}[Robust Sparse PCA]{theorem}{SPARSE PCA}\label{thm:sparse-pca}
    For an $\eps \in (0,\eps_0)$, let $T$ be an $\eps$-corrupted set of $n$ samples (in the Huber contamination model) from $\cN(0,\vec I + \rho vv^\top)$ for an unknown $k$-sparse unit vector $v \in \R^d$ and $\Omega(\eps \log(1/\eps)) < \rho < 1$. 
    There exists an algorithm that, given corruption rate $\eps$, spike strength $\rho$, a sparsity parameter $k \in \N$, and dataset $T$  with $n:= |T| \geq \frac{k^2\log d}{\eps^2} \polylog(1/\eps)$ many samples,
    computes an estimate $\widehat{v} \in \R^d$ such that 
    with probability at least $0.9$: (i)  $\|\widehat{v}\widehat{v}^\top - vv^\top \|_{\fr} = O(\eps/\rho)$
    and (ii) $\widehat{v}^\top\vec \Sigma \widehat{v} \geq \left(1 - O\left({\eps^2}/{\rho}\right)\right)\|\vec \Sigma\|_\op$ for $\vec \Sigma:= \vec I + \rho vv^\top$.
    Moreover, the algorithm runs in $\poly(nd)$-time.
\end{restatable}
Similarly, the error guarantee of \Cref{thm:sparse-pca} 
is optimal (for the considered range of $\rho$) up to a constant factor, significantly 
improving on~\cite{XuCM13,BDLS17,DKKPS19-sparse}. 
Notably, the sample complexity dependence on $k^2$ is
necessary among computationally efficient algorithms, 
even without outliers~\cite{BerthetR13}.
Even in the dense setting, \Cref{thm:sparse-pca} provides the first polynomial-time algorithm with $O(\eps)$ error; in comparison, the only existing algorithm \cite{DKKLMS18-soda} uses quasipolynomial time to get $O(\eps)$ error for \Cref{def:robust-sparse-pca}.
We additionally highlight that the approximation factor 
of $ (1 - O(\eps^2/\rho) )$ for $\frac{\widehat{v}^\top\vec\Sigma\widehat{ v}}{\|\vec \Sigma\|_\op}$ improves upon 
the known $1 - O(\eps \log(1/\eps))$ guarantees 
achieved without the spike structure~\cite{JamLT20}.
Finally, for sparse linear regression we show:
\begin{restatable}[Robust Sparse Linear Regression]{theorem}{LINEARREGRESSION}\label{thm:lin-regression}
    For $\eps \in (0,\eps_0)$, let $T$ be an $\eps$-corrupted set of $n$ samples (in the Huber contamination model) from $P_{\beta,\sigma}$ for an unknown $k$-sparse regressor $\beta \in \R^d$ and $\|\beta\|_2 = O(\sigma)$, and $\sigma\in \R_ +$.
    There exists an algorithm that, given $\eps$, $k$, and a dataset $T$  with $n:= |T| \geq \frac{k^2\log d}{\eps^2}\polylog(1/\eps)$ many samples
    computes an estimate $\widehat{\beta} \in \R^d$ such that $\|\widehat{\beta} - \beta \|_{2} = O(\sigma\eps)$ 
    with probability at least $0.9$.   Moreover, the algorithm runs in $\poly(nd)$-time.
\end{restatable}
Similar to our previous results, the above error  is optimal up to a constant, improving upon \cite{BDLS17,LiuSLC20}. The dependence on $d,\eps$ in the sample complexity is similarly nearly optimal for efficient algorithms~\cite{brennan2020reducibility}.
The restriction on the norm of $\beta$ is rather mild because of existing algorithm from \cite{LiuSLC20} which already achieves error $O(\sigma \eps \log(1/\eps))$  in polynomial time 
(but with sample complexity depending logarithmically on the initial norm); thus, we could simply use \cite{LiuSLC20} as a warm start; see \Cref{rem:norm} for further details.

\subsection{Our Techniques}\label{sec:techniques}

At a high-level, we adapt the $O(\eps)$ error algorithm of \cite{DKKLMS18-soda} to the sparse setting, using ideas from \cite{BDLS17,DKKPS19-sparse}.

We start by explaining the standard filtering algorithms that achieve $\eps \sqrt{\log(1/\eps)}$ error. Let $\mu'$ and $\vec \Sigma'$ be the empirical mean and the empirical covariance of the (corrupted) data, respectively.
Algorithms for robust mean estimation detect outliers by searching for atypical behaviors in $\vec \Sigma'$. 
Particularly, if $v^\top  \vec \Sigma' v \geq 1 + C\eps \log(1/\eps)$ for some direction $v$, then one can filter points using projections $|v^\top (x - \mu')|^2$, with the guarantee of removing more outliers than inliers (on average).
This additional $\log(1/\eps)$ factor is necessary here 
because the $\eps$-tail of $(v^\top X)^2$ for $X \sim \cN(0,\vec I)$ is at $\log(1/\eps)$ (and that of $|v^\top X|$ at $\sqrt{\log(1/\eps)}$); without this factor, the algorithm might remove too many inliers. 
Consequently, when the algorithm stops, there could be directions $v$ with variance $1 + \Theta(\eps\log(1/\eps))$ such that the $\eps n$ outliers remain $\Omega(\sqrt{\log(1/\eps)})$ far from the $v^\top \mu$, leading to a total error of $\Omega(\eps \sqrt{\log(1/\eps)})$ in the algorithm's output.

To improve this error to $O(\eps)$, \cite{DKKLMS18-soda} makes the following key observation.
If there are  $r$ (orthogonal) directions $v_1,\dots,v_r$ all with variance bigger than $1 + C \eps$,
then (i) either $r$ is small, implying that a brute-force approach can be used to learn the mean optimally in this $r$ dimensional space (in the orthogonal space, the sample mean would already be $O(\eps)$ close), or
(ii) $r$ is large, in which case, it is unlikely for an inlier to have large projections along $r$ of them \emph{simultaneously} (formalized by the Hanson-Wright inequality), {thus permitting us to remove more outliers for large $r$}.
Choosing $r = \Theta(\log(1/\eps))$, both (i) {runs in polynomial time} and (ii) removes {sufficiently many outliers}. 
The resulting algorithm thus filters until the $r$-th {largest} eigenvalue is at most $1 + O(\eps)$, thereby decomposing the data into an $r$-dimensional space $V$ and its complement $V^\perp$ such that the sample mean on $V^\perp$ has error $O(\eps)$, while the brute force approach on $V$ also incurs $O(\eps)$ error and runs in polynomial time.
As a final step, the algorithm adds these two orthogonal estimates.

Adapting this approach to the sparse regime in a sample-efficient manner 
requires that we filter outliers only along sparse directions $v$, 
which immediately hits the roadblock that maximizing $v^\top \vec \Sigma' v$ over sparse directions $v$ is computationally hard.
Thus, robust sparse estimation requires relaxing the objective $v^\top \vec \Sigma 'v = \langle \vec \Sigma', vv^\top \rangle$ for computational efficiency (while still being sample-efficient).
The relaxation of \cite{BDLS17} maximizes $\langle \vec \Sigma', \vec A  \rangle$ over PSD  matrices $\vec A$ with unit trace and bounded entry-wise $\ell_1$ norm, which is a semidefinite program.\footnote{The relaxation amounts to ignoring the rank constraint and relaxing $\ell_0$ norm to $\ell_1$ norm.}
If the maximum is larger than $1 + \Omega(\eps \log(1/\eps))$ with maximizer $\vec A^*$, one can filter out points $x$ with large score $x^\top \vec A^*x$. 
Since the filter relies only on a single ``direction'' $\vec A$, this approach is inherently limited to $\eps \sqrt{\log(1/\eps)}$ error.

Adapting \cite{DKKLMS18-soda}'s approach to the relaxation of \cite{BDLS17} is challenging.
Promisingly, it is plausible that one can filter along $r$ orthogonal ``directions'' $\vec A_1,\dots,\vec A_r$ (sample-efficiently) such that if their average score is $1 + \Omega(\eps)$, {then one may remove enough outliers.
} 
Consequently, at the end of filtering, we can identify $r$ ``directions'' $\vec A_1,\dots,\vec A_r$ such that all other orthogonal feasible $\vec A$'s would have small score{, i.e., $\langle \vec A, \vec \Sigma' \rangle = 1+O(\eps)$}.
At this point, however, the analogy of $\vec A$'s being a ``direction'' breaks down. There is no natural decomposition of the data using $\vec A$'s into a low-dimensional space $V$ and its orthogonal space $V^\perp$, such that the variance in $V^\perp$ is  $1 + O(\eps)$ (so that the sample mean is $O(\eps)$ close on $V^\perp$). 

We instead consider a different relaxation from \cite{DKKPS19-sparse} that maximizes $\langle \vec \Sigma' -\vec I, \vec A \rangle$ over $k^2$-sparse unit Frobenius norm matrices $\vec A$.\footnote{The relaxation ignores the rank, symmetry, and PSD constraints.} 
Their key observation was that the resulting relaxation is both sample and computationally efficient. %
Since they filtered along a single $\vec A$, their algorithm could not achieve $o(\eps \sqrt{\log(1/\eps))}$ error.
However, since their relaxations consider \emph{sparse} matrices $\vec A$, they naturally lead to a  decomposition of coordinates: support of $\vec A$ and the rest of the coordinates. 
Inspired by \cite{DKKLMS18-soda}, we extend \cite{DKKPS19-sparse}'s approach as follows:  
We start with an empty set of coordinates $H$ and find the sparse matrix $\vec A_1$ that maximizes $\langle \vec \Sigma' - \vec I, \vec A_1 \rangle$.
If the maximum is larger than $1 + \Omega(\eps)$, we add the support of $\vec A_1$ to $H$, and proceed to find $\vec A_2$ that maximizes $\langle (\vec \Sigma' - \vec I)_{H^\complement}, \vec A_2 \rangle$, where $(\vec \Sigma' - \vec I)_{H^\complement}$ is zero on the coordinates in $H$.
We continue until we have either (i) identified $r$ such $\vec A_i$'s, each with score $1 + \Omega(\eps)$, in which case we filter similarly 
to \cite{DKKLMS18-soda}; or (ii) we have identified a small set of coordinates, $H$, such that the sample mean is $O(\eps)$ accurate on $H^\complement$.
This still leaves the task of estimating the mean on the coordinates in $H$: although brute force approach is not possible on $H$, we can invoke the \emph{dense} algorithm from \cite{DKKLMS18-soda} on $H$ using fresh samples since
$|H| = \tilde{O}(k^2)$.

For sparse PCA, we provide a novel reduction that reduces robust Gaussian PCA to robust (approximate) Gaussian mean estimation. 
It is crucial here that we maintain the (approximate) Gaussianity in the latter because robust mean estimation for generic subgaussian distributions incurs $\omega(\eps)$ error.
In fact, even for \emph{dense} robust PCA, our algorithm is the first 
polynomial-time algorithm to achieve $O(\eps)$ error.
Recall that our goal is to estimate $v$ from corrupted samples of $X \sim \cN(0,\vec I + \rho vv^\top)$.
Given an initial rough approximation $w$ of the spike $v$, we focus on estimating the correction $z := v - w$. We decompose $z$ as $z := z' + z_{\perp}$, where $z'$ is parallel to $w$ and $z_{\perp} \perp w$; the challenge lies in estimating $z_\perp$.
Our key observation concerns the conditional distribution of (uncorrupted) samples projected orthogonally to $w$, conditioned on  $w^\top x =a $, which we denote by $X_a^\perp$. It turns out that the distribution of $X_a^\perp$ is Gaussian, with mean proportional to $z_\perp$, and approximately isotropic covariance.
 Although this insight reduces sparse PCA to (Gaussian) sparse mean estimation, this does not directly lead to an algorithm, 
because we cannot simulate this conditional sampling 
exactly (even in the outlier-free setting).
We combine our insight with a template from \cite{diakonikolas2023near}, which overcomes similar challenges in   linear regression.

\subsection{Related Work}
\label{sec:related-work}
Our work lies in the field of robust statistics, initiated in the 1960s~\cite{Tuk60,Huber64}.
We refer the reader to \cite{diakonikolas2023algorithmic} for a comprehensive overview and discuss the most relevant works below.

     \paragraph{Robust PCA} 
Principal Component Analysis has been studied extensively in the outlier-robust setting using a variety of algorithmic approaches, such as robustly estimating the covariance matrix first and maximizing certain robust variance measures (see \cite{croux2000principal, xu2010robust, candes2011robust} and the references therein). 
\cite{XuCM13} gave the first efficient algorithm that overcomes prior work's challenges stemming from high-dimensions. 

    \paragraph{Robust Sparse Estimation} 
In high-dimensional statistics, sample sizes that scale with the dimension $d$ can quickly become overwhelming. However, a smaller sample size is possible under additional structural assumptions such as sparsity.
In the context of mean estimation of distributions with light tails, the folklore sample size of $d$ is replaced by $k$ when the mean is known to be $k$-sparse~\cite{Hastie15}.
Similar improvement is known for the robust version of the problem, where $\eps$-fraction of the samples is corrupted \cite{BDLS17,DKKPS19-sparse,CDGW19,diakonikolas2022robust}.
Focusing on sparse mean estimation, subsequent works have proposed further extensions such as \cite{DiaKLP22} for heavy-tailed distributions,  \cite{diakonikolas2022robust} for light-tailed distributions with unknown covariance matrix, \cite{CheDKGGS21} for non-convex first order methods (also see \cite{ZhuJS20}), and \cite{DiaKKPP22-neurips,ZenShe22} for list-decodable estimation.
\cite{LiuSLC20} extended the work of \cite{BDLS17} to sparse linear regression.

PCA, has also been studied under sparsity. For the uncorrupted case, \cite{berthet2013optimal,wang2016statistical} provided optimal information-theoretic bounds as well as evidence through reductions to planted clique problem that efficient algorithms might require quadratically more samples. \cite{croux2013robust} provided a sparse adaptation of earlier techniques for robust sparse PCA that showed improved performance in simulations. \cite{BDLS17} and \cite{DKKPS19-sparse} studied theoretically the formulation of the problem as stated in this paper.
Finally, guarantees have been developed for robust sparse linear regression too (see \cite{chen2013robust} for early work on this problem). \cite{BDLS17} gave sample-efficient and poly-time algorithm for the task but with an error that scales with $\|\beta\|_2$, the norm of the unknown regressor. 
\cite{LiuSLC20} removed this dependence on $\|\beta\|_2$, resulting in nearly optimal error.

     \paragraph{Huber Contamination}
    The Huber contamination model of \Cref{def:huber} is the prototypical model under which the study of robust statistics was initiated. 
Since then, stronger models have been used, such as the ``total variation model'' where the samples come i.i.d. from a distribution that is $O(\eps)$-away from the original one in TV-distance, or
the so called ``strong contamination model'', where a set of samples are drawn i.i.d. from the original distribution and then a computationally unbounded adversary is allowed to inspect them and edit arbitrarily $\eps$-fraction of them, potentially breaking independence between samples. Information-theoretically, for all of these models, the optimal error for Gaussian $k$-sparse mean estimation is $\Theta(\eps)$ using $k \log(d)/\eps^2$ (see, e.g., \cite{diakonikolas2023algorithmic}). 
However, the different models play a role when computational efficiency is considered.\footnote{For other tasks, there even may be statistical differences:  \cite{CanHLLN23} has shown that the sample complexity for these two models may be different (for testing problems).}
Prior works on robust sparse mean estimation that obtain $O(\eps \sqrt{\log(1/\eps)})$ error (such as \cite{BDLS17,DKKPS19-sparse})  succeed under the strong contamination model (and thus also Huber contamination model). However, there is evidence that with $\poly(k)$ samples, even in the total variation model, it is computationally hard to remove the $\sqrt{\log(1/\eps)}$ factor from the error.
This evidence comes in the form of Statistical Query (SQ) lower bounds \cite{DKS17-sq} (which transfers to the low-degree polynomials model due to the equivalence between them \cite{brennan2021statistical}). Finally, we emphasize that we can not relax the Gaussianity assumption to generic sub-Gaussianity, since, even under univariate Huber contamination, the information-theoretic optimal error is $\eps \sqrt{\log(1/\eps)}$ for sub-Gaussian distributions.

    \paragraph{Robust Sparse Estimation with Unknown Covariance} One generalization of \Cref{def:robust-sparse-mean-estimation} is to consider Gaussian $k$-sparse mean estimation when the covariance $\vec \Sigma$ is not necessarily identity and unknown to the algorithm. The information-theoretic limit for this case remains unchanged, apart from the fact that now the error naturally needs to scale with the size of the covariance, i.e., it becomes $O(\eps )\sqrt{\|\vec \Sigma\|_\op}$.
    However, we do not know of a polynomial time algorithm to achieve $O(\eps )\sqrt{\|\vec \Sigma\|_\op}$ error, while the currently best known polynomial-time algorithm~\cite{diakonikolas2022robust} achieves a larger error of $\eps \polylog(1/\eps)\sqrt{\|\vec \Sigma\|_\op}$ error (i.e., off by a $\polylog(1/\eps)$ factor) with $\poly(k/\eps)$ samples.
    Achieving the optimal $O(\eps )\sqrt{\|\vec \Sigma\|_\op}$ error in polynomial time is not even known in the dense setting: the current fastest algorithm runs in quasi-polynomial time~\cite{DKKLMS18-soda}.

\section{Preliminaries}

\textbf{Notation} We denote $[n] := \{1,\ldots,n\}$.
For $w : \R^d \to [0,1]$ and a distribution $P$, we use $P_w$ to denote the weighted by $w$ version of $P$, i.e., the distribution with pdf $P_w(x) = w(x)P(x)/\E_{X \sim P}[w(X)]$.
We use $\mu_P, \vec\Sigma_P$ for the mean and covariance of $P$.
When the vector $\mu$ is clear from the context, we use $\overline{ \vec \Sigma}_{P}$ to denote the second moment matrix of $P$ centered with respect to $\mu$, i.e., $\overline{\vec \Sigma}_{P}:=\E_{X \sim P}[(X-\mu)(X-\mu)^\top]$.
We use $\|\cdot\|_2$ for $\ell_2$ norm of vectors and $\|\cdot\|_0$ for the number of non-zero entries in a vector. 
For a (square) matrix $\vec A$, we use $\trace(\cdot)$,  $\|\cdot\|_\op,$ and $\|\cdot\|_\fr$ for trace, operator, and Frobenius norm. We use  $\langle \vec A,\vec B \rangle := \tr(\vec A^\top \vec B) = \sum_{i,j} A_{j,i} B_{i,j}$ for the inner product between matrices.
If $H \subset [d]$ and $v \in \R^d$, we denote by $(v)_H$ the  vector $x$ restricted to the entries in $H$. 
We use $\polylog()$ to denote a quantity that is poly-logarithmic in its arguments and  $\tilde{O}$, $\tilde{\Omega}(),\tilde{\Theta}$ to hide such factors.\looseness=-1

\begin{definition}[Sparse Euclidean Norm]\label{def:sparsenorm}
    For $x \in \R^d$ and $k \in [d]$, we define  
    \begin{align*}
        \|x\|_{2,k} := \sup_{v: \|v\|_2 \leq 1, \|v\|_0 \leq k} v^\top   x \;.
    \end{align*}
\end{definition}

\begin{definition}[Sparse Frobenius and Operator Norm]
    For a $d \times d$ matrix $\vec A$, for $i \in [d]$, let $A_i$ denote the rows of $\vec A$. 
    We define 
    $\|\vec A\|_{\fr,k,k} := \sqrt{\max_{S \subseteq [d]: |S|=k} \sum_{i \in S}\|A_i\|_{2,k}^2}$.
    For a matrix $\vec A$, we define $\|\vec A\|_{\op,k}:= \sup_{v: \|v\|_0 \leq k, \|v\|_2 \leq 1} \|  \vec A v \|_2$.
\end{definition}
\noindent Note we have the alternative variational definition below:

\begin{fact}[Variational definition]\label{fact:variational}
    $\|\vec A\|_{\fr,k,k}  = \max_{\vec B} \langle \vec A, \vec B \rangle$ where the maximum is taken over all matrices $\vec B$ with $\|\vec B\|_\fr =  1$ that have $k$ non-zero rows, each of which has at most $k$ non-zero  elements.
    Moreover, a maximizer  $\vec B$ of this variational formulation can be found in $\poly(d,k)$ time given $\vec A$.
\end{fact}
\begin{proof}
Let $\vec M$ be the square matrix that is equal to 1 for  each $(i, j)$ for which $\vec A_{i,j}^2$ is witnessed in $\max_{S \subseteq [d]: |S|=k} \sum_{i \in S}\|A_i\|_{2,k}^2$, and $0$ otherwise. 
Also, let $ [\vec A \odot \vec M]_{i,j} := \vec A_{i,j} \vec M_{i, j} $.

Then, 
\begin{align*}
\|\vec A\|_{\fr,k,k} 
&= \sqrt{\max_{S \subseteq [d]: |S|=k} \sum_{i \in S}\|A_i\|_{2,k}^2} = \| \vec A \odot \vec M \|_\fr = \max_{\vec V, \|\vec V\|_\fr =1} \sum_{i,j} \vec A_{i,j} \vec M_{i,j} \vec V_{i,j}.
\end{align*}
Since $\vec M_{i,j}$ is non-zero only on $k$ rows, 
and $k$ elements in each of these rows, the expression above is equivalent to $\max_{\vec B} \langle \vec A, \vec B \rangle$ 
where the maximum is taken over all matrices $\vec B$ with $\|\vec B\|_\fr =  1$ that have $k$ non-zero rows, each of which has at most $k$ non-zero  elements. 
Given $\vec A$, we can construct the mask $\vec M$, and setting $\vec B := (\vec A \cdot \vec M) / \|\vec A \cdot \vec M\|_\fr$ achieves the maximum value.  
\end{proof}

\noindent We also note the following inequality between $\|\cdot\|_{\fr,k,k}$ and $\|\cdot\|_{\op,k}$ norms:\looseness=-1
\begin{restatable}{fact}{NORMINEQUALITY}\label{fact:h_1_lower_bound}
We have that
$\|\vec A\|_{\op,k} \leq \| \vec A \|_{\fr,k,k}$. 
\end{restatable}
\begin{proof}
This is true because $vv^\top$ and $-vv^\top$ for any $k$-sparse unit vector $v$  has Frobenius norm 1 and has at most $k$ non-zero rows, 
each of which has at most $k$ non-zero entries.
\end{proof}

\subsection{Deterministic Conditions on Inliers}
Recall that for a distribution $D$ and a weight function $w:\R^d \to [0,1]$, the distribution $D_w$ denotes the weighted (and appropriately normalized) version of $D$ using $w$. 
We further use $\mu_{D_w}$ to denote the mean of $D_w$.
\begin{restatable}[$(\eps, \alpha, k)$-goodness]{definition}{GOODNESS}
\label{DefModGoodnessCond}
For $\eps \in (0,1/2)$, $\alpha > 0$ and $k \in \N$,
 we say that a distribution $G$ on $\R^d$ is $(\eps, \alpha,k)$-good with respect to $\mu \in \R^d$,  
 if the following are satisfied:
\begin{enumerate}[leftmargin=*,label=(\arabic*), itemsep=1pt,topsep=1pt]

 \item For all  $w:\R^d {\to} [0,1]$ with $\E_{X \sim G}[w(X)] \geq 1 -  \alpha$:
 
 \begin{enumerate}[label=(\arabic{enumi}.\alph*),topsep=0pt]
  \item (Mean) \label[cond]{cond:mean}  
  $\| \mu_{G_w} - \mu \|_{2,k} \lesssim \alpha \sqrt{\log(1/\alpha)}$.
  
  \item (Covariance) \label[cond]{cond:covariance}   
  $\| \overline{\vec \Sigma}_{G_w} - \vec I \|_{\op,k} \lesssim \alpha \log(\frac{1}{\alpha})$,
  where $\overline{\vec \Sigma}_{G_w}{:=} \frac{1}{\E\limits_{X \sim G}[w(X)]} \E\limits_{X \sim G}[w(X)(X  -  \mu)(X -  \mu)^\top]$.  
\end{enumerate}
 
\item (Tails of \emph{sparse} degree-2 polynomials) \label[cond]{cond:polynomials}
If $\vec A \in \R^{d\times d}$ is a matrix with at most $k^2$ non-zero elements,  $\|\vec A\|_\fr \leq \sqrt{\log(1/\eps)}$ and $\| \vec A \|_\op \leq 1$, 
then the polynomial $p(x) := (x-\mu)^\top \vec A (x-\mu)-\tr(\vec A)$ satisfies: 
\begin{enumerate}
    \item \label{cond:poly-conc} $\E_{X\sim G}[p(X) \1(p(X) > 100 \log(1/\eps) ] \leq     \eps$.
    \item \label{cond:hw} $\Pr_{X \sim G}[p(X) > 10\log(1/\eps)] \leq \eps$.
    \item \label{cond:extra} $\E_{X\sim G}[p(X) \1(h(x) > 100 \log(1/\eps) ] \leq     \eps$ for all $h(x)$ of the form $h(x)=\beta + v^\top(x-\mu)$ where $|\beta|\leq 1$ and $v$ is $k$-sparse and unit norm.
\end{enumerate}
    \item \label{cond:linear_tails} $\Pr_{X \sim G}[ |v^\top (X -\mu)| \geq 40 \log(1/\eps))] \leq \eps$, for all $k$-sparse unit norm vectors $v \in \R^d$.
\end{enumerate}
\end{restatable}
We will focus on regime $\alpha=\Theta(\eps/\log(1/\eps))$. 
We show in  \Cref{lem:samples} in \Cref{app:prelims} that if $G$ is the uniform distribution on a set of $(\frac{k^2}{\eps^2})\polylog(\frac{d}{\eps})$ i.i.d.\ samples from $\cN(\mu,\vec I)$, then, with high probability, $G$ is $(\eps, \Theta(\eps/\log(1/\eps)),k)$-good with respect to $\mu$.

\subsection{Certificate Lemma}

The following lemma shows that if the covariance 
with respect to the weighted distribution $P_w$ is close to the identity along $k$-sparse directions, 
then the mean of $P_w$ is close to $\mu$ in $(2, k)$-norm. Its proof is similar to prior work but we include it for completeness in \Cref{app:deterministic}.

\begin{restatable}[Certificate Lemma]{lemma}{CERTIFICATE}
\label{LemCertiAlt_restated}
Let $0<\alpha<\eps<1/4$. 
Let $P = (1-\eps)G + \eps B$ be a mixture of distributions, 
where $G$ satisfies \Cref{cond:mean,cond:covariance} of \Cref{DefModGoodnessCond} 
with respect to $\mu \in \R^d$. 
Let $w : \R^d \to [0,1]$ be such that $\E_{X \sim G}[w(X)] > 1-\alpha$. 
If $\left\| \vec \Sigma_{P_w} - \vec I  \right\|_{\op,k} \leq \lambda$, then 
\begin{align*}
  \| \mu_{P_w} - \mu \|_{2,k} \lesssim  \alpha \sqrt{\log\left(\frac{1}{\alpha}\right)} + \sqrt{\lambda \epsilon} +  \epsilon +   \sqrt{\alpha \eps \log\left(\frac{1}{\alpha}\right)} .
\end{align*}
\end{restatable}
Motivated by \Cref{LemCertiAlt_restated}, the algorithm starts with weights $w(x)=1$ for all data and aims to iteratively down-weight outliers until $\left\| \vec \Sigma_{P_w} -\vec I  \right\|_{\op,k} = O(\eps)$; We also need to ensure inliers are not too much downweighted, in the sense
$\E_{X \sim G}[w(X)] > 1-\alpha$ with $\alpha =\Theta(\eps/\log(1/\eps))$).
If achieved, the total error will be $O(\eps)$, as desired.
As it will turn out, we will be able to ensure only that a large subspace has small sparse operator norm, not the entire $\R^d$.
While we can estimate $\mu$ there using \Cref{LemCertiAlt_restated}, we shall use the following estimator on the complement
subspace (with resulting sample complexity scaling with the subspace's rank).\looseness=-1

\begin{restatable}[Dense Mean Estimation \cite{DKKLMS18-soda,diakonikolas2023near}]{fact}{DENSEMEANEST}\label{fact:dense}
    There is a polynomial-time algorithm that, given parameters 
    $\eps \in (0,\eps_0), \delta \in (0,1)$ and 
    $n \geq \frac{C}{\eps^2}(d + \log(1/\delta)) \polylog(d/\eps)$ samples, for a large constant $C$, from an $\eps$-corrupted version of $\cN(\mu,\vec I)$ in the Huber contamination model, 
    computes an estimate $\widehat{\mu}$ such that $\|\widehat{\mu} - \mu \|_2 = O(\eps)$ 
    with probability at least $1-\delta$.
\end{restatable}

\subsection{Down-weighting Filter}

The filtering step of the algorithm uses the following standard procedure: It uses a weight $w(x) \in [0,1]$ to every data point $x$ and a non-negative score $\tilde{\tau}(x) \geq 0$, whose role is to quantify our belief about how much of an outlier $x$ is.
Let $G$ and $B$ denote the uniform distribution over the inliers and outliers, respectively. The uniform distribution of the entire data is denoted by $P=(1-\eps)G+
\eps B$. If $s$ is a known bound to the weighted scores of inliers, i.e., $\E_{X \sim G}[w(x)\tilde{\tau}(x)] \leq s$, then \Cref{alg:downweighting-filter} checks whether the average score over the entire dataset is abnormally large, i.e., $\E_{X \sim P}[w(x)\tilde{\tau}(x)] >s\beta$ (where $\beta>1$ is a parameter), and if so, down-weighs each point $x$ proportionally to its $\tilde{\tau}(x)$.\looseness=-1

\begin{algorithm}
\caption{Down-weighting Filter}
\label{alg:downweighting-filter}
\begin{algorithmic}[1]
\Statex \textbf{Input}: Distribution $P$ on $n$ points, weights $w(x)$, scores ${\tilde{\tau}(x) \geq 0}$, threshold $s>0$, parameter $\beta>0$.
\Statex\textbf{Output}: New weights $w'(x)$.
\Statex 
    \State Initialize $w'(x) \gets w(x)$.
    \State $\ell_{\max} \gets \max_x  \frac{\tilde \tau_{\max}}{es}$, where $\tilde \tau_{\max}:=\max\limits_{x \in \mathrm{support}(P)}\tilde \tau(x)$.
    \For{$i=1,\ldots,\ell_{\max}$}%
\If{$\E_{X \sim P}[w'(X) \tilde{\tau}(X)] >  s \beta$} \label{line:filter_cond}
    \State \label{line:down-weight} $w'(x) \gets w'(x)(1-\tilde{\tau}(x)/\max_{x: w(x)>0} \tilde \tau(x))$. \label{line:for_loop_filter}
    \EndIf
    \EndFor
    \State \textbf{return} $w'$.
	\end{algorithmic}
\end{algorithm}

The filter guarantees that it removes roughly $\beta$-times more mass from outliers than inliers.
Given the preceding discussion after \Cref{LemCertiAlt_restated}, we will eventually use  $\beta= \Theta(\eps/\alpha)=\Theta(\log(1/\eps))$.
The filter is now standard (see, e.g., \cite{dong2019quantum,diakonikolas2023near} for proofs).
\begin{restatable}[Filtering Guarantee]{lemma}{FILTERING}\label{lem:filtering}
Let $P = (1-\eps) G + \eps B$ be a mixture of distributions supported on $n$ points and $\beta>1$. If $(1-\eps) \E_{X \sim G}[w(X) \tilde{\tau}(X)] < s$,  then the new weights $w'(x)$ output by \Cref{alg:downweighting-filter} satisfy:
\begin{align*}
    (1-\eps) \E\limits_{X \sim G}[w(X) - w'(X) ] < \frac{\eps}{\beta - 1}  \E\limits_{X \sim B}[w(X) - w'(X) ],
\end{align*}
and the filter can be implemented in $O(n \log (\frac{\tilde \tau_{\max}}{s})))$-time.
\end{restatable}

\section{Robust Sparse Mean Estimation}\label{sec:mean_est_main}

This section is organized as follows: \Cref{sec:alg_desc} provides the necessary definitions and pseudocode for describing the algorithm, \Cref{sec:proof_overview} gives a high-level proof sketch of the algorithm's correctness, and \Cref{sec:mean_est_appendix} provides the full proof.

\subsection{Algorithm Description and Pseudocode}\label{sec:alg_desc}

As mentioned in \Cref{sec:techniques}, the algorithm (stated in \Cref{alg:main}) consists of two parts, summarized below:
\begin{itemize}
    \item (First phase) First, a loop that iteratively finds sparse and orthogonal maximizers $ \vec A_i,\ldots,\vec A_{r}$ of $\vec \Sigma_w  -   \vec I$ for $r=\log(1/\eps)$, which are used to filter outliers, until the ``average variance'' along these $\vec A_i$'s drops to  $O(\eps)$ (cf. Line \ref{line:loop})
    \item (Second phase) After the loop, the algorithm identifies a set $H$ of $k^2r$-many coordinates  informed by the final $\vec A_i$'s (cf.\ Line~\ref{line:coords}).
    Algorithm then splits the space $\R^d$ into $[d]\setminus H$ and $H$, and finds an $O(\eps)$ approximation of $\mu$ for both subspaces separately. For the former, it uses the empirical mean in those coordinates (which ought to be accurate because of  \Cref{LemCertiAlt_restated}), and for the latter, it employs a dense mean estimator (cf.\ \Cref{fact:dense}).
\end{itemize}

\begin{algorithm}[h]
	\caption{Robust Sparse Mean Estimation}
	\label{alg:main}
	\begin{algorithmic}[1]
        \Statex \textbf{Input}: Set of points  $T_0=\{ x_i \}_{i\in[n]}$ and $\eps>0$. 
        \Statex \textbf{Output}: A vector $\widehat \mu \in \R^d$.
        \State Let $C$ be a sufficiently large constant, and $r:= \log(\frac{1}{\eps})$.
        \State $T \gets \textsc{Preprocessing}(T_0,\eps,k)$. 
        \label{line:preprocess} \hfill\Comment{cf. \Cref{fact:preprocess}}
        \State \label{line:naive_prune} Initialize $w(x) \gets \1( \|x - \mu_T \|_2 \leq  10\sqrt{d} \log(d/\eps) )$.
        \State Let $P$ be the uniform distribution on the set $T$,  $P_w$ be the weighted by $w$ version of $P$ (with pdf $P_w(x)=P(x)/\E_{X \sim P}[w(X)]$) , $\mu_w := \E_{X \sim P_w}[w(X)]$ and $\vec \Sigma_w{:=} \E_{X \sim P_w}[(X - \mu_w)(X - \mu_w)^\top]$ the weighted mean and covariance of $P$. \footnotemark\label{line:notation}
        \While{$ \frac{1}{r}g_{r}(\vec \Sigma_w - \vec I) > C \eps$} \label{line:loop}  
            \State Let $\vec A_1, \ldots, \vec A_r$ be the matrices from \Cref{def:norm_def}.
            \State Define  $\tilde p(x) := (x - \mu_w)^\top \vec A (x-\mu_w)-\tr(\vec A)$,
            and $\tilde \tau(x)  =   \tilde p(x) \1(\tilde p(x)  >   200 \log(\frac{1}{\eps}))$ for $\vec A  =   \sum_{i \in [r]} \vec A_i$. \label{line:scores}
            \State 
            \label{line:filter}
            Update  $w \gets \textsc{DownweightingFilter}(P,w, \tilde \tau,  s=\eps, \beta = \log(1/\eps))$.
        \EndWhile

        \State For $i\in [r]$, let $H_{i} \subseteq [d]$ be sets defined in \Cref{def:norm_def} and form $H := \bigcup_{i=1}^r H_{i}$ (cf. \Cref{def:norm_def}). \label{line:coords}
        \State Run the dense mean estimator from \Cref{fact:dense} on a {fresh} data restricted to the coordinates in $H$, to obtain $\widehat \mu_1 \in \R^d$ that is zero in every coordinate in $[d] \setminus H$ and satisfies $\|(\widehat \mu_1-\mu)_H\|_{2,k} = O(\eps)$. 
        \State Let $\widehat \mu_2$ be the vector that is equal to $\E_{P_w}[X]$ in the coordinates in $[d] \setminus H$ and  zero in the coordinates in $H$.
        \State \textbf{Return} $\widehat{\mu} = \widehat \mu_1 + \widehat \mu_2$. 
        
	\end{algorithmic}

\end{algorithm}
 \footnotetext{The weights $w(x)$ may change in the course of the algorithm; $\mu_w,\vec \Sigma_w$ will denote the quantity based on the latest weights.}
In the rest of the section, we formalize this high-level sketch.
Throughout the section, we will use the notation $P,\mu_w,\vec \Sigma_w,\tilde{p},\tilde{\tau}$  defined in Lines \ref{line:notation} and \ref{line:scores} of the pseudocode.
Starting with the first phase, we formally define $\vec A_1,\ldots,\vec A_r$ mentioned in the previous paragraph, and we also define what we informally referred to as ``average variance along the $\vec A_i$'s''. In particular,  $\vec A_1,\ldots,\vec A_r$ will be the matrices in \Cref{def:norm_def} below  for $\vec B=\vec \Sigma_w -\vec I$, and the ``average variance along the $\vec A_i$'s'' will be $\frac{1}{r}g_{r}(\vec \Sigma_w-\vec I)$.
\begin{definition} \label{def:norm_def}
For any matrix $\vec B$, 
we define $h_{i}(\vec B)$, $\vec A_i$, and $H_i$ for $i \in [r]$ recursively as follows.
\begin{itemize}
    \item For $i=1$, $h_{1}(\vec B) := \| \vec B \|_{\fr,k,k} = \max_{\vec A \in \cS} \langle \vec A,\vec B \rangle$ where 
$\cS$ is the set of matrices $\vec A$ that have $\|\vec A\|_\fr=1$ and 
have at most $k$-non-zero rows, each of which has at most $k$ non-zero entries. 
Let $\vec A_{1}$ be the matrix achieving the maximum. The set $H_{1} \subseteq [d]$  denotes the rows {and columns} in which $\vec A_{1}$ has non-zero elements.
\item For $i \in \{2,\dots,r\}$, 
we recursively define $h_{i}(\vec B)$, $\vec A_{i}$, and $H_{i}$ as follows:
$h_{i}(\vec B)  =   \|\vec B'\|_{\fr,k,k}$ where $\vec B'$ is $\vec B$ after deleting (zeroing out) the rows and columns from $H_{1} \cup \cdots \cup H_{i-1}$. 
Similarly, $\vec A_{i} := \argmax_{\vec A \in \cS} \langle \vec A,\vec B' \rangle$ 
and $H_{i}$ is the non-zero rows and columns of $\vec A_{i}$.
\item 
Finally, we define $g_{r}(\vec B) = \sum_{i=1}^r h_{i}(\vec B)$.
\end{itemize}
\end{definition}
We now explain why we informally  call 
$\frac{1}{r}g_{r}(\vec \Sigma_w -  \vec I)$ the ``average variance along the $\vec A_i$'s'':
For each $i$, $h_i(\vec B)$ represents the mean of the degree-two polynomial
$(x -  \mu_w)^\top \vec A_i (x -  \mu_w) -  \tr(\vec A_i)$, representing a variance-like quantity of $x$ along $\vec A_i$. 
Formally,
(see \eqref{eq:rewriting} for the details):\looseness=-1
\begin{equation}
   g_{r}(\vec \Sigma_w - \vec I)  = \sum_{i=1}^r \E_{X \sim P_w}[(X - \mu_w)^\top \vec A_i (X - \mu_w) - \tr(\vec A)] \label{eq:rewritting_main_body}
\end{equation}

\subsection{Proof Overview of \Cref{thm:main}}\label{sec:proof_overview}
The proof of correctness consists of the following claims:
\begin{enumerate}[leftmargin=*]
\item \label{it:firsttt}For any iteration of line \ref{line:loop}, if $w(x),w'(x)$ denote the weights before and after the iteration:
\begin{enumerate}
    \item \label{it:outliers-inliers}($\log(1/\eps)$ more outliers than inliers are removed) $\E\limits_{X \sim G}[w(X)  {-}   w'(X) ]  {<}   \frac{2\eps}{\log(\frac{1}{\eps})}  \E\limits_{X \sim B}[w(X)  {- } w'(X) ]$.
    \item \label{it:massremoved}(Non-trivial mass is removed) $\E_{X \sim P}[w(X)-w'(X)] = \tilde{\Omega}(\eps/d)$. 
\end{enumerate}

    \item \label{it:empirical_good} After the loop ends, $\| (\widehat{\mu}_2-\mu)_{[d]\setminus H} \|_{2,k} = O(\eps)$.\footnote{Recall $(x)_H$ denotes the vector $x$ restricted to $H\subset [d]$.}
\end{enumerate}
\Cref{it:massremoved} means that the algorithm terminates after $\tilde{O}(d/\eps)$ iterations (since no outliers are left at that point), \Cref{it:outliers-inliers} means that $\E_{X \sim G}[w(X)]\geq 1 - 3\eps/\log(1/\eps)$ is an invariant condition throughout the loop, and \Cref{it:empirical_good} states that the empirical mean is accurate on the coordinates from $[d]\setminus H$.

Before proving these claims, we show how they imply \Cref{thm:main}. 
We decompose $\mu$ into $\mu_1+\mu_2$ for $\mu_1 = (\mu)_H$ and $\mu_2:= (\mu)_{[d]\setminus H}$. 
By triangle inequality and definition of $\widehat{\mu}$, we have  $\|\widehat{\mu} - \mu \|_{2,k} \leq \|\widehat{\mu}_1 - \mu_1 \|_{2,k} + \|\widehat{\mu}_2 - \mu_2 \|_{2,k}$.
We bound each of these terms by $O(\eps)$.
The inequality $\|\widehat{\mu}_1  -   \mu_1   \|_{2,k} \leq O(\eps) $ corresponds to the guarantee of the dense estimator (\Cref{fact:dense}), run 
 on a fresh dataset, when restricted to coordinates in $H$. 
\Cref{fact:dense} gives an estimator with $O(\eps)$ error and sample complexity  
$\frac{1}{\eps^2}(|H|  +   \log(1/\delta))\polylog(d/\eps)$.
Since $|H| \leq k^2\log(1/\eps)$, it satisfies \Cref{thm:main}.
The term $\left\| \widehat{\mu}_2{ -} \mu_2\right\|_{2,k}$ is equal to  $\| (\widehat{\mu}_2-\mu)_{[d]\setminus H} \|_{2,k}$ and thus $O(\eps)$ by \Cref{it:empirical_good}.

We now sketch the proofs of \Cref{it:outliers-inliers,it:massremoved,it:empirical_good} used above.\looseness=-1

 \paragraph{Proof of \Cref{it:empirical_good}} Once \Cref{it:massremoved}  is shown, it implies that $\E_{X \sim G}[w(X)]\geq 1 - 3\eps/\log(1/\eps)$ and thus \Cref{it:empirical_good} becomes a straightforward application of the Certificate \Cref{LemCertiAlt_restated} with $\alpha=3\eps/\log(1/\eps)$ and $\lambda=C\eps$ to agree with the stopping condition of line \ref{line:loop} (for this we need to show that the condition of line \ref{line:loop} implies $\|(\vec \Sigma_w - \vec I)_{([d]\setminus H)\times ([d]\setminus H)} \|_{\op,k} = O(\eps)$; see \Cref{cl:final} in \Cref{sec:mean_est_appendix} for the details).

\paragraph{Proof of \Cref{it:outliers-inliers}} We want to use \Cref{lem:filtering} with $\beta=\log(1/\eps)$ and $s=\eps$.
To apply the lemma, we need to show that $\E_{X \sim G}[w(X)\tilde{\tau}(X)] \leq \eps$. This looks like \Cref{cond:poly-conc} of the goodness conditions (\Cref{DefModGoodnessCond}) but the difference is that $\tilde{\tau}(x)$ centers the point around $\mu_{w}$ 
instead of $\mu$ that is used in $\tau(x)$.
While these centering issues are easily dealt with when $\vec A$ is PSD and no sparsity constraints are present, 
our setting requires additional technical work. We defer the full proof to \Cref{sec:mean_est_appendix}, sketching the steps here.

  Let $\tilde{\tau}(x) = \tilde{p}(x)\1(\tilde p(x) > 200\log(1/\eps))$, where $\tilde p(x) = (x-\mu_w)^\top \vec A (x-\mu_w)-\tr(\vec A)$; The algorithm uses $\tilde \tau(x)$ as scores.
  Define $\tau(x) = p(x)\1(p(x) > 100 \log(1/\eps))$, with $ p(x) = (x-\mu)^\top \vec A (x-\mu)-\tr(\vec A)$, to be the ideal scores appearing in the deterministic condition (that center data around the true $\mu$). 
  Denote the difference of these polynomials by $\Delta p(x):=\tilde p(x) - p(x) = \delta_\mu^\top \vec A \delta_\mu + (x-\mu)^\top \vec A \delta_\mu + (x-\mu)^\top \vec A^\top \delta_\mu$ for $\delta_\mu:=\mu-\mu_w$.
  Using triangle inequalities, we get \looseness=-1
  \begin{align}
      &\E_{X \sim G}[w(X)\tilde{\tau}(X)] \quad \leq \quad |\delta_\mu^\top \vec A \delta_\mu| \notag\\
      &\quad +   \E_{X \sim G}[(X -  \mu)^\top \vec A \delta_\mu \1(  p(X){ >} 200 \log(1/\eps)  -   \Delta p(X))]\notag\\
      &\quad +   \E_{X \sim G}[(X -  \mu)^\top \vec A^\top \delta_\mu \1(  p(X){ >} 200 \log(1/\eps)  -   \Delta p(X))]\notag\\
      &+ \E_{X \sim G}[    p(X)  \1(  p(X) > 200 \log(1/\eps) - \Delta p(X))].\label{eq:terms}
  \end{align}
  We need to bound all {three} terms above by $O(\eps)$. For the first term,  we take advantage of the sparsity of $\vec A$ to establish:
  \begin{restatable}{claim}{SPARSECS}\label{cl:inequality_sparse_norms}
    Let $\vec A = \sum_{\ell \in [r]} \vec B^{(\ell)}$ where each $\vec B^{(\ell)}$  is a square matrix with Frobenius norm equal to one, $k$ non-zero rows, each of which has $k$ non-zero entries. Then, for any vectors $u,v$, it holds $|u^\top \vec A v |\leq r \| u \|_{2,k} \|v \|_{2,k}$.
\end{restatable}
This means that $|\delta_\mu^\top \vec A \delta_\mu| \leq \log(1/\eps) \|\delta_\mu\|_{2,k}^2$, which can eventually be bounded by $\eps$ using  \Cref{LemCertiAlt_restated} and the preprocessing of Line \ref{line:preprocess}. The second term in \eqref{eq:terms} may be broken into two terms by considering the cases $\Delta p(X) \leq 100 \log(1/\eps)$ and $\Delta p(X) > 100 \log(1/\eps)$. The latter case is a very low-probability event by \Cref{cond:linear_tails}, 
eventually bounding the relevant term by $\eps$. For the former case, we can use that $p(X) > 200 \log(1/\eps)$ is a low-probability event (by \Cref{cond:hw} of \Cref{DefModGoodnessCond}). The third term uses an identical argument. Finally, the third term in \eqref{eq:terms} is similarly split into two by taking cases for $\Delta p(X)$, and bounding each one using either \Cref{cond:poly-conc} or \Cref{cond:extra} of the \Cref{DefModGoodnessCond}. 
  
\paragraph{Proof of \Cref{it:massremoved}}
By design, the down-weighting filter only removes mass when $\E_{X \sim P}[w(X)\tilde{\tau}(X)] \geq s \beta =: \eps \log(1/\eps)$ (cf. line \ref{line:filter_cond} of \Cref{alg:downweighting-filter}). 
Thus, we first need to show that this is true throughout the loop of Line \ref{line:loop}.
To do so, we write $\E_{X \sim P}[w(X)\tilde{\tau}(X)] = \E_{X \sim P}[w(X)\tilde{p}(X)]-\E_{X \sim P}[w(X)\tilde{\tau}(X)\1(\tilde{p}(X) \leq 200 \log(1/\eps))]$. 
The first term is already roughly $g_r(\vec \Sigma_w -  \vec I)$ by \eqref{eq:rewritting_main_body},  up to a normalization of $\E[w(x)] \approx 1$,
and hence at least $\eps\log(1/\eps)$ by Line \ref{line:loop}). 
Showing that the second term is less than $g_r(\vec \Sigma_w-\vec I)/2$ involves a multi-step argument similar to the ones in the previous paragraph, which can be found in \Cref{sec:mean_est_appendix}. Once this is established, \Cref{it:massremoved} follows easily: First, by design of the down-weighting filter, $\E_{X \sim P}[w(X) - w'(X)] = \E_{X \sim P}[w(X) \tilde{\tau}(X)]/\max_{x}\tilde{\tau}(X)$. We have already shown that the numerator is $\Omega(\eps)$. The denominator is $O(d\,\polylog(d/\eps))$ since $\|x\|_2 = \sqrt{d}\log(d/\eps)$ for all points in the dataset, by Gaussian concentration.

\subsection{Proof of \Cref{thm:main}}\label{sec:mean_est_appendix}
In this section we provide the complete proof of correctness of \Cref{alg:main}.
Recall \Cref{def:norm_def}. We record a useful fact about the that definition.

\begin{fact}
    Consider the matrix $\vec A = \sum_{i=1}^r \vec A_i$ where 
    $\vec A_i$ for $i\in [r]$ are the matrices in \Cref{def:norm_def}. 
    Then, it is true that 
    (i) $\|\vec A\|_\fr = \sqrt{r}$, and 
    (ii) $\|\vec A\|_\op \leq 1$.
\end{fact}
\begin{proof}

    Proof of (i): 
    Since $\vec A = \sum_{i=1}^r \vec A_i$, the $\vec A_i$'s have 
    their non-zero entries on disjoint coordinates, 
    and $\|\vec A_i\|_\fr=1$ we have that 
    $\| \vec A \|_\fr^2 = \sum_{i=1}^r \|\vec A_i\|_\fr^2 = r$.
    
    Proof of (ii): 
    Observe that for $i \neq j$, the matrices $\vec A_i$ and $\vec A_j$ satisfy that if $(k,\ell)$ entry of $\vec A_i$ is non-zero, then the corresponding entry of $\vec A_j$ must be zero.
    As a result, by relabeling coordinates, we see that $\vec A$ can be written as a block matrix: 
    \begin{align*}
        \vec A = \begin{bmatrix}
            \vec A_1 & \mathbf{O} & \dots&  \mathbf{O}&\mathbf{O} \\
            \mathbf{O}& \vec A_2 &\dots  &\mathbf{O} &\mathbf{O} \\
            \vdots & \vdots&\vdots &\vdots& \vdots&\\
           \mathbf{O} &\mathbf{O} & \dots& \vec A_r & \mathbf{O}\\
           \mathbf{O} &\mathbf{O} & \dots& \mathbf{O} & \mathbf{O}
        \end{bmatrix} \;,
    \end{align*}
    where $\mathbf{O}$ represent a matrix with all entries set to zero; observe that $\mathbf{O}$ above could be of varing dimensions.
    Therefore, $\vec A^\top \vec A$ is equal to 
    \begin{align*}
        \vec A^\top \vec A = \begin{bmatrix}
            \vec A_1^\top \vec A_1 & \mathbf{O} & \dots&  \mathbf{O}&\mathbf{O} \\
            \mathbf{O}& \vec A_2^\top \vec A_2 &\dots  &\mathbf{O} &\mathbf{O} \\
            \vdots & \vdots&\vdots &\vdots& \vdots&\\
           \mathbf{O} &\mathbf{O} & \dots& \vec A_r^\top \vec A_r & \mathbf{O}\\
           \mathbf{O} &\mathbf{O} & \dots& \mathbf{O} & \,\,\mathbf{O}
        \end{bmatrix} \;,
    \end{align*}
Since this is a block diagonal PSD matrix, the operator norm of $\vec A^\top \vec A$ is equal to $\max_{i}\|\vec A_i^\top \vec A_i\|_\op$, which we can upper bound as $\|\vec A_i^\top \vec A_i\|_\op \leq \trace(\vec A_i^\top \vec A_i ) = \|\vec A_i\|_\fr^2 = 1$.
Therefore, $\|\vec A\|_\op \leq \sqrt{\|\vec \vec A^\top \vec A\|_\op} \leq 1$.
\end{proof}

We now provide the full proof of \Cref{thm:main}, where we state each of the three steps as separate claims and prove them individually.

\begin{proof}[Proof of \Cref{thm:main}]
We start with some notation. 
Let $T$ be the dataset after preprocessing (line \ref{line:preprocess}). 
We let $P$ be the uniform distribution over $T$, 
$G$ be the uniform distribution over the inliers of $T$ and 
$B$ the uniform distribution over the outliers. 
We can write $P$ as mixture $P = (1-\eps)G + \eps B$. 
If $w:T \to [0,1]$ denotes the weights that the algorithm maintains at a given point 
during its execution, 
we denote by $P_w$ the weighted by $w$ version of $P$.
By \Cref{lem:samples}, we have that with probability at least $1-\delta$, $G$ is $(\eps,\alpha,k)$-good with $\alpha=3\eps/\log(1/\eps)$.

We will show \Cref{thm:main} via the following steps listed as individual claims below:

\begin{claim}\label{it:filtering}
Consider the setting and notation of \Cref{thm:main,alg:main}. 
The condition $\E_{X \sim G}[w(X)] \geq 1 - 3\eps/\log(1/\eps)$ remains true 
throughout the execution of the loop in line \ref{line:loop}. 
\end{claim}

\begin{claim}\label{cl:termination}
Under the setting of \Cref{thm:main}, 
the loop of line \ref{line:loop} terminates after $\tilde{O}(d/\eps)$ time. 
\end{claim}

\begin{claim} \label{cl:final}
Consider the setting and notation of \Cref{thm:main,alg:main}. 
After the loop of line \ref{line:loop} ends, 
let $w$ be the resulting weight function, 
let $\mu_w = \E_{X \sim P_w}[X]$ be the mean of the dataset weighted by $w$, 
and denote by $H$ the set of coordinates as in line \ref{line:coords} of the pseudocode.
Then, for every $k$-sparse $v \in \R^d$, 
it holds $| v^\top( \mu_w - \mu)_{[d] \setminus H}| = O(\epsilon) \|v\|_2$.
\end{claim}

We will prove these claims at the end of the current proof. 
We first use these to show that the algorithm has error $O(\eps)$ overall:
For every $k$-sparse vector $v$ of $\R^d$, 
let $v_1$ denote the copy of $v$ that has all of its entries in the coordinates 
from $[d] \setminus H$ zeroed out and 
$v_2$ the copy of $v$ that has all of the entries in the coordinates from $H$ zeroed out. 
Similarly, decompose $\mu$ into $\mu_1+\mu_2$. 
Then, by the triangle inequality:

\begin{align*}
    \left\| \hat{\mu} - \mu  \right\|_{2,k} 
    =  \left\| \hat{\mu}_1 - \mu_1 + \hat{\mu}_2 - \mu_2  \right\|_{2,k} 
    \leq \left\| \hat{\mu}_1 - \mu_1   \right\|_{2,k} + \left\|   \hat{\mu}_2 - \mu_2  \right\|_{2,k} \;,\\
\end{align*}
To conclude the proof, we claim that 
$\left\|  \hat{\mu}_1 - \mu_1  \right\|_{2,k } = O(\eps)$ and
$\left\| \hat{\mu}_2 - \mu_2 \right\|_{2,k} =   O(\eps)$. 
    
The claim that $\left\|  \hat{\mu}_1 - \mu_1  \right\|_{2,k } = O(\eps)$
follows by the $O(\eps)$-error guarantee of the dense robust mean estimator 
run on a dataset restricted to the coordinates in $H$. 
Let $d':= |H|$ be the dimensionality of the restricted dataset. 
\Cref{fact:dense} states that in order for that estimator to achieve $O(\eps)$ error 
the dataset used needs to be of size a sufficiently large multiple of 
$\frac{1}{\eps^2}(d' + \log(1/\delta))\polylog(d/\eps)$. 
Since  $d' = |H| \leq r k = \log(1/\eps) k$, this quantity is 
$O(\frac{1}{\eps^2}(k + \log(1/\delta)))\polylog(d/\eps)$ and 
thus smaller than the sample complexity mentioned in the statement of \Cref{thm:main}.

The claim that $\left\|   \hat{\mu}_2 - \mu_2  \right\|_{2,k} = O(\eps)$ follows by \Cref{cl:final}.

We now prove \Cref{it:filtering,cl:termination,cl:final}. 

\begin{proof}[Proof of \Cref{it:filtering}]
To show this, it suffices to show that every time the weight function is updated,
$\log(1/\eps)$ more mass is removed from outliers than inliers.
We will show this inductively: 
Assume it is true for all previous rounds and 
we will show it for the current round using \Cref{lem:filtering} applied with 
$\beta = \log(1/\eps), s=\eps$ once we show 
that the lemma is applicable. 
Denote
$\lambda' := \| \vec \Sigma_w - \vec I \|_{\fr,k,k}$.
In order to show that  \Cref{lem:filtering} is applicable, 
we need to check that $\E_{X \sim G}[w(X)\tilde{\tau}(X)] \leq \eps$.
Regarding that, $\tilde{\tau}(X)$ looks like the thresholded polynomial $\tau(x)$ 
used in the deterministic \Cref{cond:polynomials} for which we know that 
$\E_{X \sim G}[w(X)\tau(X)] \leq  \eps$. 
The difference is that $\tilde{\tau}(x)$ centers the point around $\mu_{w}$ 
instead of $\mu$ that is used in $\tau(x)$, thus we need some triangle inequalities and \Cref{cl:inequality_sparse_norms}, which is shown in \Cref{app:sparse-matrix-CS}, to prove that this difference is not substantial.
  
  Let $\tilde{\tau}(x) = \tilde{p}(x)\1(\tilde p(x)> 200\log(1/\eps))$, where $\tilde p(x) = (x-\mu_w)^\top \vec A (x-\mu_w)-\tr(\vec A)$ be the scores that the algorithm uses (which center points around $\mu_w$) and $\tau(x) = p(x)\1(p(x) > 100 \log(1/\eps))$ with $ p(x) = (x-\mu)^\top \vec A (x-\mu)-\tr(\vec A)$ be the ideal scores appearing in the deterministic condition (that center things around the true $\mu$). Denote by $\Delta p(x):=\tilde p(x) - p(x) = (\mu-\mu_w)^\top \vec A (\mu-\mu_w) + (x-\mu)^\top \vec A (\mu-\mu_w) + (x-\mu)^\top \vec A^\top (\mu-\mu_w)$ the difference of the two polynomials. We have that
  \begin{align}
      \E_{X \sim G}[w(X)\tilde{\tau}(X)] 
      &\leq \E_{X \sim G}[ \tilde{\tau}(X)] \notag \\
      &= \E_{X \sim G}[\tilde p(X)  \1(\tilde p(X) > 200 \log(1/\eps))] \notag \\
      &= \E_{X \sim G}[(\Delta p(X) + p(X) ) \1( p(X) > 200 \log(1/\eps)-\Delta p(X))]  \notag \\
      &= \E_{X \sim G}[\Delta p(X) \1(  p(X) > 200 \log(1/\eps) - \Delta p(X))] \notag \\
      &\quad+ \E_{X \sim G}[   p(X)   \1(  p(X) > 200 \log(1/\eps) - \Delta p(X))]  \notag \\
      &\leq |(\mu-\mu_w)^\top \vec A (\mu-\mu_w)| \notag\\
      &\quad+ \E_{X \sim G}[(X-\mu)^\top \vec A (\mu-\mu_w) \1(  p(X) > 200 \log(1/\eps) - \Delta p(X))] \notag\\
      &\quad+ \E_{X \sim G}[(X-\mu)^\top \vec A^\top (\mu-\mu_w) \1(  p(X) > 200 \log(1/\eps) - \Delta p(X))] \notag\\
      &\quad+ \E_{X \sim G}[    p(X)  \1(  p(X) > 200 \log(1/\eps) - \Delta p(X))] \;. \label{eq:temp132}
  \end{align}
  
  We claim that each of the four terms above is at most $O(\eps)$.
Denote $$\lambda := \max_{v \in \R^d: \|v\|_2=1, \|v\|_0=k} v^\top( \vec \Sigma_w - \vec I) v$$ 
(not to be confused with $\lambda' := \| \vec \Sigma_w - \vec I \|_{\fr,k,k}$). For the first term in \eqref{eq:temp132}, we have that
  \begin{align}
      |(\mu - \mu_w)^\top \vec A (\mu - \mu_w)|
      &\leq r \|\mu-\mu_w\|_{2,k}^2 \tag{using \Cref{cl:inequality_sparse_norms}}\\
      &\lesssim r (\eps \lambda + \eps^2)  \tag{using \Cref{LemCertiAlt_restated}}\\
      &\leq r \eps \lambda' + r \eps^2   \tag{$\lambda \leq \lambda'$ by \Cref{fact:h_1_lower_bound} }\\
      &\leq  r \eps^2 \log^2(1/\eps)  \tag{using $\lambda' \lesssim   \eps \log^2(1/\eps)$}\\
      &\leq \eps, \label{eq:tmp5345}
  \end{align}
 applicable as follows:  in the second line we applied \Cref{LemCertiAlt_restated} with $\alpha=3\eps/\log(1/\eps)$  (the requirement that $\E_{X \sim G}[w(X)] \geq 1- \alpha$ of that lemma is satisfied by inductive hypothesis), and  
 the last line uses that $r:=\log(1/\eps)$.

  We now move to the second term in \eqref{eq:temp132}. The bound for the third term is identical, thus we will omit it. We have that
  \begin{align}
       &\E_{X \sim G}[ (X-\mu)^\top \vec A (\mu-\mu_w) \1(  p(X) > 200 \log(1/\eps) - \Delta p(X))] \notag\\
       &\leq    \E_{X \sim G}[  (X-\mu)^\top \vec A (\mu-\mu_w)  \1( p(X) > 200 \log(1/\eps) - \Delta p(X), \Delta p(X) \leq 100 \log(1/\eps))] \notag\\
       &\quad+ \E_{X \sim G}[  |(X-\mu)^\top \vec A (\mu-\mu_w)| \1(  \Delta p(X) > 100 \log(1/\eps))] \notag\\
       &\leq   \E_{X \sim G}[| (X-\mu)^\top \vec A (\mu-\mu_w)| \1(  p(X) > 100 \log(1/\eps))] \notag\\
       &\quad+  \E_{X \sim G}[ |(X-\mu)^\top \vec A (\mu-\mu_w)| \1(  \Delta p(X) > 100 \log(1/\eps))] \;.\label{eq:temp4324}
  \end{align}

    We now work with the two terms individually. We start with the first term of \eqref{eq:temp4324}. 
    In what follows we define the vectors $u_i:= \vec A_i(\mu-\mu_w) / \| \vec A_i(\mu-\mu_w)\|_{2,k} $ to shorten notation.
    We have the following series of inequalities (see below for step by step explanations) 
    \begin{align}
        &\E_{X \sim G}[ |(X-\mu)^\top \vec A (\mu-\mu_w)| \1(  p(X) > 100 \log(1/\eps))] \notag\\
        &\leq \sum_{i=1}^r \E_{X \sim G}[ |(X-\mu)^\top \vec A_i (\mu-\mu_w)| \1(  p(X) > 100 \log(1/\eps))] \label{eq:tmp00011}\\
        &\lesssim \sum_{i=1}^r \sqrt{\E_{X \sim G}[ |(X-\mu)^\top \vec A_i (\mu-\mu_w)|^2] } \sqrt{\Pr_{X \sim G}[p(X) > 100 \log(1/\eps)]} \label{eq:tmp00021}\\
        &= \sum_{i=1}^r \| \vec A_i ( \mu - \mu_w) \|_{2,k} \sqrt{\E_{X \sim G}[ |(X-\mu)^\top u_i|^2] }  \sqrt{\Pr_{X \sim G}[p(X) > 100 \log(1/\eps)]} \label{eq:tmp0003}\\
        &\leq \sum_{i=1}^r \|\vec A_i \|_\fr \| \mu - \mu_w \|_{2,k} \sqrt{\E_{X \sim G}[ |(X-\mu)^\top u_i|^2] }  \sqrt{\Pr_{X \sim G}[p(X) > 100 \log(1/\eps)]} \label{eq:tmp00031}\\
        &\lesssim r  (\sqrt{\lambda' \eps} +\eps)\sqrt{\Pr_{X \sim G}[p(X) > 100 \log(1/\eps)]} \label{eq:tmp0004}\\
        &\lesssim r  (\sqrt{\lambda' \eps}+\eps) \sqrt{\eps} \label{eq:tmp00041}\\
        &\lesssim  \eps^{1.5} \log^2(1/\eps) \label{eq:tmp0005}\\
        &\leq \eps \;.\label{eq:tmp0006}
    \end{align}
    Note that \eqref{eq:tmp00011} follows from the definition of $A = \sum_{i=1}^r \vec A_i$ and the triangle inequality. 
    \eqref{eq:tmp00021} uses the  Cauchy–Schwarz inequality.
    \eqref{eq:tmp0003} is a re-writing using $u_i:= \vec A_i (\mu-\mu_w) / \| \vec A_i (\mu-\mu_w)\|_{2,k} $, where the point to note is that $u_i$ is $k$-sparse (because $\vec A_i$ has only $k$ non-zero rows) and unit norm.
    \eqref{eq:tmp00031} uses the inequality $\|\vec C v \|_{2,k} \leq \|\vec C\|_\fr \|v\|_{2,k}$. This is true since $\vec C$ is a matrix with at most $k$ nonzero rows, at most $k$ nonzero entries in each row, and has Frobenius norm one. An application of \Cref{cl:inequality_sparse_norms} with $r=1$ then gives us what we want. Then, 
    \eqref{eq:tmp0004} uses that $\|\vec A_i\|_\fr \leq 1$, $\| \mu - \mu_w \|_{2,k} \lesssim\sqrt{\eps \lambda}+\eps \leq  \sqrt{\eps \lambda'}+\eps$ by \Cref{LemCertiAlt_restated} and \Cref{fact:h_1_lower_bound}, and $\E_{X \sim G}[ |(x-\mu)^\top u_i|^2] \leq 1+ \tilde{O}(\eps)\lesssim 1$ by the deterministic \Cref{cond:covariance} (these utilize the fact that $u_i$ is unit-norm $k$-sparse vector).
    \eqref{eq:tmp00041} uses that $\Pr_{X \sim G}[p(X) > 100 \log(1/\eps)] \leq \eps$ by the deterministic \Cref{cond:hw}.
    \eqref{eq:tmp0005} uses that $\lambda' \lesssim \eps \log^2(1/\eps)$ by the preprocessing step of the algorithm (cf.\ \Cref{fact:preprocess}) and also that $r=\log(1/\eps)$.

  We now move to the second term of \eqref{eq:temp4324}.  
  \begin{align}
      &\E_{X \sim G}[|(X - \mu)^\top \vec A (\mu - \mu_w)|\1(  \Delta p(X) > 100 \log(1/\eps))]  \notag\\
      &\leq r (\sqrt{\lambda' \eps}+\eps) \sqrt{\Pr_{X \sim G}[\Delta p(X) > 100 \log(1/\eps)]}   \tag{similar to steps (\ref{eq:tmp00011}-\ref{eq:tmp0004})} \\
      &\lesssim \eps^{1.5} \log^2(1/\eps)  \lesssim\eps \;,
  \end{align}
  where we bounded $\Pr_{X \sim G}[\Delta p(X) > 100 \log(1/\eps)]$ as follows: First, $\Pr_{X \sim G}[\Delta p(x) > 100 \log(1/\eps)] 
          \leq \Pr_{X \sim G}[|(X-\mu)^\top \vec A (\mu-\mu_w)| > 99 \log(1/\eps)] + \Pr_{X \sim G}[|(X-\mu)^\top \vec A^\top (\mu-\mu_w)| > 99 \log(1/\eps)]$, where this uses that $\Delta p(x) = (\mu - \mu_w)^\top \vec A(\mu - \mu_w) + (x - \mu)^\top \vec A^\top (\mu - \mu_w) + (x - \mu)^\top \vec A  (\mu - \mu_w)$, and that $|(\mu - \mu_w)^\top \vec A(\mu - \mu_w)|\leq 1 < \log(1/\eps)$ by \eqref{eq:tmp5345}. We will only focus on the first term, since the second probability is bounded identically:
  
      \begin{align}
          \Pr_{X \sim G}[|(X-\mu)^\top \vec A (\mu-\mu_w)| > 99 \log(1/\eps)] 
          &\leq \Pr_{X \sim G}\left[\left|(X-\mu)^\top \frac{\vec A (\mu-\mu_w)}{\|\vec A (\mu-\mu_w)\|_2} \right| >  \frac{99\log(1/\eps)}{\|\vec A (\mu-\mu_w)\|_2}\right] \label{eq:temp000002}\\
          &\leq \Pr_{X \sim G}\left[\left|(X-\mu)^\top \frac{\vec A (\mu-\mu_w)}{\|\vec A (\mu-\mu_w)\|_2} \right| >  99\log(1/\eps)\right]  \label{eq:temp000003}\\
          &\leq \eps  \;, \label{eq:temp000004}
      \end{align}
      where
      \eqref{eq:temp000002} divides both sides  by $\vec A (\mu-\mu_w)$;  
      \eqref{eq:temp000003} uses that  $\|\vec A(\mu-\mu_w)\|_{2} \leq \sum_{i=1}^r \|\vec A_i (\mu-\mu_w)\|_{2} = \sum_{i=1}^r   \|\vec A_i (\mu-\mu_w)\|_{2,k} \leq  \sum_{i=1}^r\|\vec A_i\|_\fr \|\mu-\mu_w\|_{2,k} \leq r\|\mu-\mu_w\|_{2,k} \lesssim r (\sqrt{\lambda' \eps} +\eps) \leq 1$, where the first step is a triangle inequality, the second step uses the fact that $\vec A_i$ has only $k$-nonzero rows, the next step uses Cauchy-Schwartz, then we use   that there are $r$ terms in the sum, $\|\vec A_i\|_\fr\leq 1$, and finally that $\|\mu-\mu_w\|_{2,k} \lesssim \sqrt{\eps \lambda'}+\eps$ by \Cref{LemCertiAlt_restated} and finally that $r=\log(1/\eps),\lambda' = O(1)$ by our preprocessing step inside the algorithm (\Cref{fact:preprocess}).
      \eqref{eq:temp000004} uses \Cref{cond:linear_tails} of \Cref{DefModGoodnessCond}, which is indeed applicable because $\vec A (\mu-\mu_w)$ is $r k$-sparse (since $\vec A$ has at most $r k$ non-zero rows).

  We are now done with all the terms in \eqref{eq:temp4324} and can now move to the last term of \eqref{eq:temp132}:
  \begin{align*}
      &\E_{X \sim G}[  p(X)  \1(  p(X)  > 200 \log(1/\eps) - \Delta p(X))]\\
      &\leq \E_{X \sim G}[  p(X)  \1(  p(X) > 200 \log(1/\eps) - \Delta p(X) , \Delta p(X) < 100\log(1/\eps))] \\
      &\quad+ \E_{X \sim G}[  p(X) \1(  \Delta p(X) > 100\log(1/\eps))] \;.
  \end{align*}
  The first term above is $ \E_{X \sim G}[  p(X) \1(   p(X) > 100\log(1/\eps))] =  \E_{X \sim G}[ \tau(x)]$ which is at most $O(\eps)$ by our deterministic \Cref{cond:polynomials}. 
  The term 
  $\E_{X \sim G}[  p(X) \1(  \Delta p(X) > 100\log(1/\eps))]$ is bounded by $\eps$ using  \Cref{cond:extra} of the deterministic \Cref{cond:polynomials} (the application is very similar to the application of \Cref{cond:linear_tails} of \Cref{DefModGoodnessCond} in equations \eqref{eq:temp000002}-\eqref{eq:temp000004}). 
    \end{proof}

\begin{proof}[Proof of \Cref{cl:termination}]

    We will show that in every iteration of the while loop of line \ref{line:loop}, $\E_{X \sim P}[w(X)]$ is reduced by at least $\tilde \Omega(\eps/d)$, where $w(x)$ are the weights that the algorithm maintains and $P$ is the uniform distribution over the input dataset. Since initially $\E_{X \sim P}[w(X)] = 1$, this would mean that after at most $\tilde O(d/\eps)$ iterations the weight from all the outliers would have been reduced to zero. We will finally show that this would trigger the stopping condition of line \ref{line:loop} and cause the algorithm to terminate.

    We start with the first claim, showing that $\E_{X \sim P}[w(X)]$ gets non-trivially reduced in every round. 
    Fix an iteration of the algorithm and denote by $w(x)$ the weights at the start of that round and by $w'(x)$ the weights after that round ends. By design of the filtering algorithm (see line \ref{line:down-weight} of \Cref{alg:downweighting-filter}), the mass removed is 
    \begin{align}
        \E_{X \sim P}[w(X) - w'(X)] = \frac{\E_{X \sim P}[w(X) \tilde{\tau}(X)]}{\max_{x: w(x)>0}\tilde{\tau}(X)} \;.
    \end{align}

    We will show that the denominator is $O(d\log^2(d/\eps))$, and that the numerator is $\Omega( \eps \log(1/\eps))$.

    Regarding the denominator, let $\mu_T$ denote the vector from line \ref{line:naive_prune} of the algorithm (the vector used in the na\"ive pruning step). For every point with $w(x)>0$ it holds $\tilde \tau(x) \leq |(x - \mu_w)^\top \vec A (x - \mu_w)| \leq \|\vec A\|_\op \|x - \mu_w\|_2^2 \leq  \|x - \mu_w\|_2^2 \lesssim  \|x - \mu\|_2^2 +\|\mu - \mu_T\|_2^2 + \| \mu_T - \mu_w \|_2^2 \lesssim d \log^2(d/\eps)$, where the fact that every point in the dataset is within $10\sqrt{d} \log(d/\eps)$ from $\mu_T$ (by the pruning done in line \ref{line:naive_prune} of the algorithm).

    Regarding the numerator, let $\lambda'' := g_{r}(\vec \Sigma_w - \vec I)$. We will show that the numerator is $\Omega(\lambda'')$. 
    Note that as long as the while loop of line \ref{line:loop} has not been terminated, $\lambda'' \gg \log(1/\eps)\eps$ by design. 
    We rewrite the numerator:
    \begin{align}\label{eq:temp42344}
        \E_{X \sim P}[w(X) \tilde{\tau}(X)] = \E_{X \sim P}[w(X) \tilde{p}(X)] - \E_{X \sim P}[w(X) \tilde{p}(X) \1(\tilde{p}(X) \leq 200\log(1/\eps))  ]
    \end{align}
    We first show that, after re-normalizing, the first term at least $0.5 g_{r}(\vec \Sigma_w - \vec I)$:

    \begin{align}
        \E_{X \sim P}[w(X) \tilde{p}(X)]
        &= \E_{X \sim P}[w(X)] \E_{X \sim P_w} [\tilde{p}(X)] \notag \\
        &= \E_{X \sim P}[w(X)]\E_{X \sim P_w}[  \langle \vec A, (X- \mu_w)(X- \mu_w)^\top - \vec I  \rangle ]\notag \\
        &=\E_{X \sim P}[w(X)]\left \langle \vec A, \E_{X \sim P_w}[ (X- \mu_w)(X- \mu_w)^\top] - \vec I  \right\rangle \notag \\
        &= \E_{X \sim P}[w(X)] \langle \vec A, \vec \Sigma_w - \vec I  \rangle \notag \\
        &= \E_{X \sim P}[w(X)]\sum_{i=1}^r \langle \vec A_i, \vec \Sigma_w - \vec I  \rangle \tag{see \Cref{def:norm_def}}\\
        &=  \E_{X \sim P}[w(X)]\sum_{i=1}^r h_{r}(\vec \Sigma_w- \vec I) \tag{see \Cref{def:norm_def}}\\
        &= \E_{X \sim P}[w(X)]g_{r}(\vec \Sigma_w - \vec I) \label{eq:rewriting}\\
        &\geq  0.5 \lambda'' \;, \notag 
    \end{align}
    where we used \Cref{it:filtering} to obtain that $\E_{X \sim P}[w(X)] \geq 1/2$.

    In the reminder, we show that the second term in \Cref{eq:temp42344} is at most $0.002\lambda''$. First, we decompose into inliers and outliers:
    \begin{align*}
        \E_{X \sim P}&[w(X) \tilde{p}(X) \1(\tilde{p}(X) \leq 200\log(1/\eps))  ]\\
        &= (1-\eps)\E_{X \sim G}[w(X) \tilde{p}(X) \1(\tilde{p}(X) \leq 200\log(1/\eps))  ]\\
        &+ \eps \E_{X \sim B}[w(X) \tilde{p}(X) \1(\tilde{p}(X) \leq 200\log(1/\eps))  ] \;.
    \end{align*}
    We again treat each term individually. For the second term (the one due to outliers), 
    we have the following due to the indicator function
    \begin{align*}
        \eps \E_{X \sim B}[w(X) \tilde{p}(X) \1(\tilde{p}(X) \leq 200\log(1/\eps))  ]
        \leq  \eps ( 200\log(1/\eps)))  \ll \lambda'' \;,
    \end{align*}
    where we used that by design of line \ref{line:loop} of the algorithm $\lambda'' > C r \eps$ with $r:=\log(1/\eps)$ and $C$ being a sufficiently large constant.

    For the term due to inliers, we can say the following: Denote by $\Delta p(x) = \tilde p(x) - p(x) = (\mu - \mu_w)^\top \vec A (\mu - \mu_w) + (x-\mu)^\top \vec A (\mu-\mu_w) + (x-\mu)^\top \vec A^\top (\mu-\mu_w)$. Then we will establish the following bounds (see below for the explanations of each step):
    \begin{align}
        &(1-\eps)\E_{X \sim G}[w(X) \tilde{p}(X) \1(\tilde{p}(X) \leq 200\log(1/\eps))  ] \label{eq:temp00000001}\\
        &\leq (1-\eps)\E_{X \sim G}[w(X) \tilde{p}(X) ]  \label{eq:temp00000002}\\
        &=  (1-\eps)\E_{X \sim G}[w(X)  p(X)  ] + (1-\eps)\E_{X \sim G}[w(X) \Delta p(X) ]  \label{eq:temp00000003}\\
        &\leq \eps + (1-\eps)\E_{X \sim G}[w(X) \Delta p(X) ]  \label{eq:temp00000004}\\
        &\leq \eps + (1-\eps)\E_{X \sim G}[w(X)  (\mu - \mu_w)^\top \vec A (\mu - \mu_w)] \notag \\
        &\quad + (1-\eps)\E_{X \sim G}[w(X)(x-\mu)^\top \vec A (\mu-\mu_w)] + (1-\eps)\E_{X \sim G}[w(X)(x-\mu)^\top \vec A^\top (\mu-\mu_w)]  \label{eq:temp00000005}\\
        &\leq O(\eps) +   (\mu - \mu_w)^\top \vec A (\mu - \mu_w)  + (\mu_G-\mu)^\top \vec A (\mu-\mu_w) +(\mu_G-\mu)^\top \vec A^\top (\mu-\mu_w)  \label{eq:temp00000006}\\
        &= O(\eps)  \label{eq:temp00000009}\\
        &\leq 0.002\lambda'' \,.\label{eq:temp00000010}
    \end{align}
    We explain the steps below:
    \eqref{eq:temp00000002} uses that the indicator only removes non-negative terms.
    \eqref{eq:temp00000003} decomposes into inliers and outliers.
    \eqref{eq:temp00000004} bounds the average value of the polynomial over inliers as follows (we will use the notation $p_{\vec A}(x) = (x-\mu)^\top \vec A (x-\mu)-\tr(\vec A)$ for clarity):
    \begin{align}
      \E_{X \sim G}[ w(X) p_{\vec A}(X) ] 
      &=  \E_{X \sim G}[w(X)  p_{\vec A}(X)  ] \notag\\
      &=   \E_{X \sim G}[ p_{\vec A}(X)  ]   -  \E_{X \sim G}[(1-w(X))  p_{\vec A}(X)] \notag\\
      &=   \E_{X \sim \cN(\mu,\vec I)}[ p_{\vec A}(X) ] + O(\eps)   -  \E_{X \sim G}[(1-w(X))  p_{\vec A}(X) ] \notag \\
      &= O(\eps) + \E_{X \sim G}[(1-w(X))  p_{-\vec A}(X) ] \label{eq:finalline}
    \end{align}
    where in the second to last line we used that for all
degree-2 polynomials $g$ with at most $k^2$ terms 
  $\left| \E_{X \sim G}[g(X)] - \E_{Z \sim \cN(\mu,\vec I)}[g(Z)]   \right| \leq \eps \sqrt{\Var_{Z \sim \cN(\mu,\vec I)}[g(Z)]}$, which can be found in which can be found in Lemma 4.3 in \cite{DKKPS19-sparse}.
    \eqref{eq:finalline} line we renamed $-\vec A$ to $\vec A'$. Then, decomposing the remaining term into the large and small part using indicator functions, we have that
    \begin{align*}
        \E_{X \sim G}[(1-w(X))  p_{-\vec A}(X) ]
        &= \E_{X \sim G}[(1-w(X))  p_{-\vec A}(X)\1(p_{-\vec A}(X) \leq 100 \log(1/\eps))  ] \\
        &+ \E_{X \sim G}[(1-w(X))  p_{-\vec A}(X) \1(p_{-\vec A}(X)  \geq 100 \log(1/\eps))  ] \\
        &\lesssim \E_{X \sim G}[(1-w(X)) 100 \log(1/\eps) ]
        + \eps \\
        &\lesssim \eps 
    \end{align*}
    where in the second to last step we used that $\E_{X \sim G}[1-w(X)] \lesssim  \eps/\log(1/\eps)$ by \Cref{it:filtering}, and also that $\E_{X \sim G}[(1-w(X))  p_{-\vec A}(X)\1(p_{-\vec A}(X)  \geq 100 \log(1/\eps))  ] \leq \eps $ by the deterministic condition \ref{cond:polynomials}.
    
    Back to explaining the sequence of bounds in \eqref{eq:temp00000001}-\eqref{eq:temp00000010}, the bound used in \eqref{eq:temp00000009} can be proved exactly as in \eqref{eq:tmp5345}
    and \eqref{eq:temp00000010} uses that $\lambda'' > C r \eps$ because of line \ref{line:loop} of the algorithm.

    We have thus completed the argument that $\tilde \Omega(\eps/d)$ mass is removed in every round of the loop of line \ref{line:loop}. To conclude the proof of  \Cref{cl:termination}, it remains to show that after $\tilde O(d/\eps)$ iterations the algorithm will necessarily terminate. First note that after that many iterations, there cannot be any outliers left. Moreover, by \Cref{it:filtering} a large fraction of inliers still remains in the dataset (i.e., $\E_{X \sim G}[w(X)] \geq 1 - 3\eps/\log(1/\eps)$). We claim that under this setting, $g_{r}(\vec \Sigma_w - \vec I) = \E_{X \sim G}[w(X)  \tilde p(X) ] \lesssim \eps \log(1/\eps)$ causing the stopping condition of line \ref{line:loop} to trigger. We show the bound as follows:

    \begin{align*}
        \E_{X \sim G}[w(X)  \tilde p(X) ]
        &\leq O(\eps) +   (\mu - \mu_w)^\top \vec A (\mu - \mu_w)  + (\mu_{G_w}-\mu)^\top \vec A (\mu-\mu_w) + (\mu_{G_w}-\mu)^\top \vec A^\top (\mu-\mu_w) \\
        &\lesssim  \eps + r \|\mu - \mu_w\|_{2,k}^2 + r \|\mu - \mu_w\|_{2,k}\|\mu_{G_w} - \mu_w\|_{2,k}
    \end{align*}
    where the first inequality is as in \eqref{eq:temp00000002}-\eqref{eq:temp00000006}. Finally
    to bound each term by $O(\eps)$ one can use the deterministic \Cref{cond:mean} and \eqref{eq:tmp5345}.

\end{proof}

\begin{proof}[Proof of \Cref{cl:final}]
    The claim follows by \Cref{LemCertiAlt_restated}. We show how to apply this lemma: By the definition of the stopping condition of the algorithm, upon termination we have that $g_{r}(\vec \Sigma_w - \vec I) \leq O(r\eps)$, or, equivalently $\frac{1}{r}\sum_{i=1}^r h_{i}(\vec \Sigma_w - \vec I) \leq O(\eps)$. This means that $h_{r }$, as the smallest term in the sum, must be $O(\eps)$. Since $h_{r }(\vec \Sigma_w - \vec I)$ is the $(\fr,k,k)$ norm of the matrix after deletion of the elements in $([d] \setminus H) \times ([d] \setminus H)$ (cf. line \ref{line:coords} of pseudocode), it follows that $\|(\vec \Sigma_w - \vec I)_{([d] \setminus H) \times ([d] \setminus H)} \|_{\fr,k,k} \leq h_{r }(\vec \Sigma_w - \vec I) \leq O(\eps)$. Combining with \Cref{fact:h_1_lower_bound}, we have that $\sup_{\norm{v}_2 = 1, \norm{v}_0 = k} v^\top ((\vec \Sigma_w - \vec I)_{([d] \setminus H) \times ([d] \setminus H)} ) v = O(\eps)$. We also recall that the dataset is ($\eps,\alpha,k$)-stable with $\alpha = 3\eps/\log(1/\eps)$ and, because of \Cref{it:filtering}, $\E_{X \sim G}[w(X)] \geq 1 - \alpha$ holds when exiting the main loop of the algorithm. All of these mean that we can apply \Cref{LemCertiAlt_restated} with $\lambda = O(\eps)$ and $\alpha = 3\eps/\log(1/\eps)$ to obtain that $\left| v_2^\top (\hat{\mu}_2 - \mu_2 )  \right| \leq \|v_2\|_2^2 O(\eps)$ for any vector $k$-sparse vector $v_2$ supported on $[d]\setminus H$.
\end{proof}

\end{proof}

\subsection{Proof of \Cref{cl:inequality_sparse_norms}}\label{app:sparse-matrix-CS}

We restate and prove the following inequality

\SPARSECS*
\begin{proof}
    First, $u^\top \vec A v = \sum_\ell u^\top \vec B^{(\ell)} v$. Consider a single matrix $\vec B^{(\ell)}$ from the sum and denote by $b^{(\ell)}_i$ for $i \in [d]$ the rows of $\vec B^{(\ell)}$ (only $k$ of them are non-zero and each has at most $k$ non-zero elements). We have the following:
    \begin{align}
         u^\top \vec B^{(\ell)} v &= u^\top {\vec B}^{(\ell)} v \notag \\
         &=  \sum_{i,j} u_i v_j  ({\vec B}^{(\ell)} )_{i j} \notag  \\
         &=  \sum_{i } u_i \sum_j v_j  ({\vec B}^{(\ell)} )_{i j}  \notag \\
         &\leq  \sum_{i } u_i  \| v \|_{2,k}  \| {b}_i^{(\ell)} \|_2  \label{eq:tmp52343}\\
         &\leq \| u \|_{2,k}  \| v \|_{2,k}  \| {\vec B}^{(\ell)} \|_\fr \label{eq:tmp52344}\\
         &=  \| u \|_{2,k}  \| v \|_{2,k} \;, \label{eq:finaleq}
    \end{align}
    where  the step in \eqref{eq:tmp52343} used Cauchy–Schwarz inequality along with the fact that at most $k$-entries are non zero in the $i$-th row of $\vec B^{(\ell)}$, and \eqref{eq:tmp52344} used Cauchy–Schwarz again along with the fact that there are at most $k$ non-zero rows in $\vec B^{(\ell)}$.

    Finally, summing over all terms in \eqref{eq:finaleq} concludes the proof.
\end{proof}

\section{Robust Sparse PCA and Linear Regression}
\subsection{Robust Sparse PCA}\label{sec:pca_main_body}
In this section, we show \Cref{thm:sparse-pca} via a novel reduction to mean estimation 
(which also implies new results for the dense setting).
A natural first attempt is to consider the following existing reduction to mean estimation (see, e.g., 
\cite{DKKPS19-sparse}): 
the mean of $\mathrm{vec}(XX^\top  -   \vec I)$ for $X \sim D$, 
where $\mathrm{vec}$ denotes the operator that converts matrices to vectors by stacking its rows, 
is exactly $\rho vv^\top$, which is also $\poly(k)$ sparse. 
However, the distribution of $\mathrm{vec}(XX^\top  -   \vec I)$ is not Gaussian 
but rather a second power of a Gaussian, thus existing mean estimators would only yield $O(\eps \log(1/\eps))$ error. 
We propose a different reduction, which does not lose this $\log(1/\eps)$ factor. 
Let $w$ be a unit vector that is an $\O(\eps \sqrt{\log(1/\eps)})$-approximation 
of the spike $v$ (e.g., using \cite{BDLS17}).
Denote the projection of $X$ onto the subspace orthogonal to $w$  by $\mathrm{Proj}_{w^\perp}(X)$.{\footnote{For a vector $u$, we use $\mathrm{Proj}_{u^\perp} (\cdot)$ to denote the projection operator on the null space of $u$.}
Our key idea is that $\mathrm{Proj}_{w^\perp}(X)$ conditioned on $w^\top x  =   \alpha$ can give information about $\mathrm{Proj}_{w^\perp}(v-w)$, i.e., the correction $v -  w$ in that subspace. 
We prove the following in \Cref{sec:appendix_pca}.
\looseness=-1

\begin{restatable}{claim}{CONDITIONAL}\label{claim:conditional}
    Let $X \sim \cN(0,\vec I+\rho v v^\top)$ be a random variable from the spiked covariance model and $w$ be a unit vector. Let $Z = \mathrm{Proj}_{w^\perp}(X)$, the projection of $X$ onto the subspace perpendicular to $w$. For $\alpha \in \R$, let $G_\alpha$ denote the distribution of $Z$ conditioned on $w^\top X = \alpha$. Then   $G_\alpha = \cN(\tilde{\mu}, \tilde{\vec \Sigma})$  with  $\tilde{\mu}=\frac{\rho (w^\top v) \alpha }{1+\rho (w^\top v)^2} \bar{v}$ and $ \tilde{\vec \Sigma} = \vec I + \frac{\rho }{1+ \rho (w^\top v)^2}\bar{v}\bar{v}^\top$, where $\bar{v}:= \mathrm{Proj}_{w^\perp} (v) =  v-(w^\top v)w$.
\end{restatable}

\looseness=-1 We use the result above to estimate $\mathrm{Proj}_{w^\perp}(v)$, and our final estimate, $\widehat{v}$, shall be $\widehat{v}_1 + \widehat{v}_2$, where $\widehat{v}_1$ estimates   $\mathrm{Proj}_{w}(v) = (w^\top v)w$ and $\widehat{v}_2$ estimates $\mathrm{Proj}_{w^\perp}(v)$.
Importantly, the mean of $G_\alpha$ is a scaled version of $\mathrm{Proj}_{w^\perp}(v)$, and thus if $z \approx \mu_{G_\alpha}$, we could use $\widehat{v}_2 = z( \rho \alpha(w^\top v))/(1 + \rho (w^\top v)^2)$.
However, since $w^\top v$ is unknown, we need to estimate it from data; we also need it to estimate $\widehat{v}_1$.
Note that $1+\rho (w^\top v)^2$ is the variance of $X^\top w$. Thus we can use one-dimensional (robust) variance estimation algorithm to find $y$ such that $|y - (w^\top v)^2| = O(\eps/\rho)$. This leads to \Cref{alg:pca}: 

\begin{algorithm}[H]
	\caption{Reduction from PCA to Mean Estimation}
	\label{alg:pca}
	\begin{algorithmic}[1]
  
        \State Find unit vector  $w$ such that $\|ww^\top -   vv^\top\|_\fr = O(\eps\sqrt{\log(1/\eps)}/\rho)$. \label{line:warmstart}
        \Comment{e.g. \cite{BDLS17}}
        \State Find $y$:  $|y -  (w^\top v)^2 | = O(\eps/\rho)$. \label{line:var_est} \hfill\Comment{by robustly estimating the variance of $(w^\top x)$  (see \Cref{cl:applicability})}
        \State Fix an $\alpha=\Omega(1)$ and find $z$ with $\|z - \mu_{G_\alpha}\|_2=O(\eps)$, where $G_\alpha$ is the conditional distribution from \Cref{claim:conditional}.\label{line:cond_est} \hfill\Comment{e.g., using \Cref{alg:main} (see \Cref{cl:applicability} for details)}

        \State Return $\widehat{v} = z\frac{1+ \rho y}{\rho \sqrt{y} \alpha} + w \sqrt{y}$.
	\end{algorithmic}

\end{algorithm}

We show that the final error, i.e., $\|\widehat{v} - v \|_2 \leq \|z\frac{1+ \rho y}{\rho \sqrt{y} \alpha}- \mathrm{Proj}_{w^\perp}(v)\|_2 + |\sqrt{y} - w^\top v |$, is $O(\eps/\rho)$ using the guarantees of the two aforementioned estimators (i.e., that $|y - (w^\top v)^2|=O(\eps)$ and $\|z - \mu_{G_\alpha}\|_2 = O(\eps) $); 
see \Cref{cl:closeness} at the end of this section for the detailed calculations. 
\looseness=-1

    To complete an overview of the proof of \Cref{thm:sparse-pca}, we need to explicitly show how to obtain $z$ in Line \ref{line:cond_est} of \Cref{alg:pca} using our sparse mean estimator, \Cref{thm:main}. An obvious issue is that we cannot simulate samples from $G_\alpha$ using $X\sim \cN(0,\vec I + \rho vv^\top)$ by rejection sampling---let alone that is at most $O(\eps)$ corrupted given $\eps$-corrupted $X$--- because the probability that a sample has $w^\top x = \alpha$ is zero. 
    To overcome this, we use insights from \cite{diakonikolas2023near} and relax this procedure by instead conditioning on samples to be in a thin interval $I$ around $\alpha$. 
    The resulting pseudocode is as follows:
\begin{enumerate}
        \item Draw $\alpha$ uniformly from $[-(1+\rho),1+\rho]$ and define the interval $I := [\alpha-\ell,\alpha+\ell]$ for $\ell = 1/\log(1/\eps)$.
        \item $T' = \{\mathrm{Proj}_{w^\perp}(x)  : \text{$x \in T$ and $w^\top x \in I$} \}$.
        \item Let $z$ be the output of \Cref{alg:main} run on $T'$.
    \end{enumerate}
    Although conditioning on an interval increases the probability and thus permits efficient rejection sampling,
the downside is that the conditional distribution on $w^\top x \in I$ is no longer a Gaussian distribution, but instead a continuous mixture of Gaussians:
\begin{align*}
    G_I(z) = \frac{\int_{\alpha' \in I}G_{\alpha'}(z) \Pr_{X \sim \cN(0, \vec I+\rho v v^\top)}[w^\top X = \alpha'] \d \alpha'}{\Pr_{X \sim \cN(0,\vec I+\rho v v^\top)}[ w^\top X \in I]}   \;.
\end{align*}
The mean of the mixture, $\mu_{G_I}$, may shift away from $\mu_{G_\alpha}$, 
and thus we need the length $\ell$ of $I$ to be small enough so that the shift is $O(\eps)$. 
Moreover, $G_I$ is not Gaussian, and thus \Cref{thm:main} is not applicable in a black-box manner. 
However, for small $\ell$, it is close enough to a Gaussian so that the deterministic conditions of \Cref{DefModGoodnessCond} hold with respect to $\mu_{G_\alpha}$ (the details are deferred to \Cref{sec:appendix_pca}).

\looseness=-1
To ensure applicability of our mean estimator, it remains to show that 
the fraction of outliers in the conditional dataset $T'$ is $O(\eps)$. 
This is why the center $\alpha$ of the interval needs to be chosen randomly (as in \cite{diakonikolas2023near}); otherwise, the outlier distribution could happen to have all outliers $x$ satisfy $w^\top x = \alpha$. 
To show that $T'$ is $O(\eps)$-corrupted, it suffices to check that the probability of an outlier $x$ 
satisfying $w^\top x \in I$ divided by the probability that an inlier $x'$ having $w^\top x' \in I$ is at most $O(1)$.  
Since $I$ is chosen independently of everything, 
we can imagine that the outlier $x$ is fixed and only $I$ is drawn randomly. 
Let us  examine only the case where $w^\top x \in [-2(1+\rho),2(1+\rho)]$ (because otherwise  $w^\top x \not\in I$). 
The probability that $w^\top x \in I$ is then the ratio of the length of $I$ to the length of the interval $[-2(1+\rho),2(1+\rho)]$, i.e.,  $O(\ell/(1+\rho))$. 
Regarding the inliers $x'$, we can use the same trick to imagine that $I$ is fixed and 
the inlier $x'$ is drawn from $\cN(0, I + \rho vv^\top)$. 
Note that $w^\top x' \sim \cN(0,\tilde{\sigma}^2)$ with $\tilde{\sigma}^2 := 1+\rho (w^\top v)^2$. 
Since $I \subseteq [-3 \tilde{\sigma}^2, 3\tilde{\sigma}^2]$, 
the Gaussian distribution behaves approximately uniformly there and thus the probability that $w^\top x\in I$ is $\Omega(\ell/\tilde{\sigma})$, which is also $\Omega(\ell/(1+\rho))$ by using $|w^\top v|^2 = \Omega(1)$. 
The details of this are again deferred to \Cref{sec:appendix_pca}.

We conclude this section with the formal proof of the error guarantee of \Cref{alg:pca} given the error guarantee of the mean estimator. 

\begin{claim}\label{cl:closeness}
    Under the assumptions of \Cref{thm:sparse-pca} and assuming that the estimators in Lines \ref{line:warmstart},\ref{line:var_est} and  \ref{line:cond_est} in \Cref{alg:pca} exist, we have that $\| \hat{v} - v \|_2 = O(\eps/\rho)$.
\end{claim}
\begin{proof}
    Let $e_1,\ldots,e_d$ denote the standard orthonormal basis of $\R^d$.
    By a rotation of the space, we can assume without loss of generality that $w$ is aligned with $e_d$, i.e., $w= \|w\|_2 e_d$.  We denote by $\overline{v} = (v_1,\ldots,v_{d-1})$ the projection of $v$ to the subspace orthogonal to $w$ and by $v_d = w^\top v$ the projection of $v$ in the direction of $w$.

    Some useful observations for later on are the following: Note that, by \Cref{fact:fr-to-eucl}, the fact that we start with $\|ww^\top - vv^\top\|_\fr = O(\eps\sqrt{\log(1/\eps)}/\rho)$ in the algorithm implies that $\|w-v\|_2 = O(\eps\sqrt{\log(1/\eps)}/\rho)$, which also means that $1 \geq w^\top v \geq 1- O(\eps^2\log(1/\eps)/\rho^2)$. Since $w^\top v=v_d$ (by our rotation assumption) and $\|v\|_2=1$, the previous discussion means that: 
    \begin{align}
       &\|\bar{v}\|_2 = \sqrt{1 - (w^\top v)^2} =  O(\eps\sqrt{\log(1/\eps)}/\rho)  
         \, \text{ and } 1- O(\eps^2\log(1/\eps)/\rho^2) \leq v_d \leq 1 \,. \label{eq:v_d_bound}
    \end{align}

    Now, let $x \sim \cN(0,\vec I+ \rho v v^\top)$ be a random vector coming from our spiked covariance model. We want to consider the distribution of $(x_1,\ldots,x_{d-1})$ conditioned on $ w^\top x = \alpha$ (note that by our rotation assumption this is equivalent to conditioning on $x_d= \alpha$). By \Cref{claim:conditional}, this conditional distribution is   
        $\cN\left(\frac{\rho v_d \alpha }{1+\rho v_d^2}\overline{v}, \vec I + \frac{\rho }{1+ \rho v_d^2} \overline{v}\overline{v}^\top   \right)$.

    The error of our final estimator is 
    \begin{align*}
        \|\hat{v} - v \|_2 \leq \left\|z\frac{1+\rho y}{\rho \sqrt{y} \alpha} - \bar{v}  \right\|_2 + |\sqrt{y}-v_d| \;.
    \end{align*}
    The second term is $O(\eps/\rho)$ by line \ref{line:var_est} of the pseudocode. To see this, note that $|y-v_d^2|=O(\eps/\rho)$ implies $|\sqrt{y}-v_d|= \frac{|y-v_d^2|}{\sqrt{y} + v_d} = O(\eps/\rho)$ by using $v_d \geq 1/2$.
    For the first term, we have the following:
    \begin{align}
        \left\|z\frac{1+\rho y}{\rho \sqrt{y} \alpha} - \bar{v}  \right\|_2
        &\leq  \left\|z\frac{1+\rho y}{\rho \sqrt{y} \alpha} - z\frac{1+\rho v_d^2}{\rho v_d \alpha}  \right\|_2  + \left\| z\frac{1+\rho v_d^2}{\rho v_d \alpha} - \bar{v}  \right\|_2 \notag\\
        &\leq \|z\|_2 \left| \frac{1+\rho y}{\rho \sqrt{y} \alpha} -    \frac{1+\rho v_d^2}{\rho v_d \alpha} \right|
        + \frac{1+\rho v_d^2}{\rho v_d \alpha} \left\|z - \bar{v}\frac{\rho \alpha v_d}{1+\rho v_d^2}    \right\|_2 \label{eq:twoterms1} \\
        &\lesssim \rho \left| \frac{1+\rho y}{\rho \sqrt{y} \alpha} -  \frac{1+\rho v_d^2}{\rho v_d \alpha} \right| + O(\eps/\rho) \;, \label{eq:twoterms}
    \end{align}
    where the first line is a triangle inequality, and the last line used the following: 
    First, the factor $({1+\rho v_d^2})/({ \rho v_d \alpha})$ is $O(1/\rho)$ because of $\alpha = \Theta(1)$, $\rho=O(1)$, and   $1\geq v_d \geq 1/3$ (by \eqref{eq:v_d_bound}). Also, by the mean estimation guarantee (line \ref{line:cond_est} of the pseudocode), we have that $\|z - \bar{v}\frac{\rho \alpha v_d}{1+\rho v_d^2} \|_2 = O(\eps)$, which bounds the last term in \eqref{eq:twoterms1}. Lastly, the previous two imply that $\|z\|_2 \leq O(\rho)$: 
\begin{align*}
    \left\| z \right\|_2 &\leq \left\| z - \bar{v}\frac{\rho \alpha v_d}{1+\rho v_d^2}\right\|_2 + \left\|\bar{v}\frac{\rho \alpha v_d}{1+\rho v_d^2} \right\|_2 \\
    &\lesssim \eps + \rho  \| \bar{v} \|_2 \tag{mean estimation guarantee, $\alpha=\Theta(1),v_d\leq 1$}\\
    &\lesssim \eps + \rho  O(\eps \sqrt{\log(1/\eps)}/\rho) \tag{by \eqref{eq:v_d_bound}}\\
    &\lesssim \eps \sqrt{\log(1/\eps)} \lesssim \rho \;. \tag{by assumption}
\end{align*}
    This is because $z$ is $O(\eps)$-close to $\bar{v}\frac{\rho \alpha v_d}{1+\rho v_d^2}$, whose norm can be checked to be $O(\eps \sqrt{\log(1/\eps)})$ using that $\alpha=\Theta(1), v_d = \Theta(1)$ and $\|\bar{v}\|_2 = O(\eps\sqrt{\log(1/\eps)}/\rho)$ (by \eqref{eq:v_d_bound}) and finally $\eps \sqrt{\log(1/\eps)} \lesssim \rho$ by assumption.

    We now bound the remaining term in \eqref{eq:twoterms}. We know that $|y-v_d^2|=O(\eps/\rho)$, thus we also have $|\sqrt{y}-v_d|=O(\eps/\rho)$. Let us write $\sqrt{y} = v_d + \eta$, $y = v_d^2 + \eta'$ for some $|\eta| = O(\eps/\rho)$, $|\eta'| = O(\eps/\rho)$. Using this and doing some algebra, the term in \eqref{eq:twoterms} is
    \begin{align*}
        \rho\left| \frac{1+\rho y}{\rho \sqrt{y} \alpha} -  \frac{1+\rho v_d^2}{\rho v_d \alpha} \right|
        &= \frac{\rho}{\rho \alpha} \left| \frac{1+\rho v_d^2 + \rho \eta'}{v_d+\eta} - \frac{1+\rho v_d^2}{v_d} \right| \\
        &\leq   \left| \frac{\rho \eta' v_d - \eta - \rho v_d^2 \eta}{v_d(v_d+\eta)} \right|  \tag{$\alpha = \Theta(1)$}\\
        & \lesssim \eps/\rho \;,
    \end{align*}
    where in the last step we used $\rho=O(1)$, $ |v_d| = \Theta(1)$ (from \eqref{eq:v_d_bound}) and $|\eta| = O(\eps/\rho)$, $|\eta'| = O(\eps/\rho)$ to show that every term is $O(\eps/\rho)$.
    
\end{proof}

\subsection{Robust Sparse Linear Regression}

We conclude with \Cref{thm:lin-regression}, which follows a reduction to mean estimation from \cite{diakonikolas2023near}. 
Their reduction seamlessly extends to the sparse setting considered in this paper, 
thus we describe it only briefly.

\looseness=-1As a first step, using \cite{LiuSLC20} as preprocessing, we may assume that $\|\beta\|_2 \lesssim \sigma \eps \log(1/\eps)$ while using at most $O(k^2\log(d) \polylog(1/\eps)/\eps^2)$ many samples.
Analogously to \Cref{claim:conditional},  \cite[Claim 4.1]{diakonikolas2023near} shows the following:
let $Q_a$ denote the conditional distribution of $X$, 
conditioned on $y = \alpha$ for $(X,y) \sim P_{\beta,\sigma}$ in \Cref{def:robust-sparse-linear-regression}, 
then $Q_a \sim \cN( (\alpha/\sigma_y^2) \beta, \vec I - \beta \beta^\top/ \sigma_y^2 )$ for $\sigma_y^2 := \sigma^2 + \|\beta\|_2^2$.
Since $Q_a$ is an approximately isotropic Gaussian with a \emph{sparse} mean, 
we can hope to estimate it with $O(\eps)$ error in a sample-efficient way by \Cref{thm:main}.
Finally, to obtain \Cref{thm:lin-regression},  one needs to use similar tricks as in the last section 
(such as conditioning on a random interval of an appropriate length instead of a fixed point).
Fortunately, all of these approximations suffice to get $O(\eps)$ error as in \cite{diakonikolas2023near}.
The final algorithm is as follows:
\begin{algorithm}[h]
	\caption{Robust Linear Regression}
	\label{alg:regression}
	\begin{algorithmic}[1]
        \State \textbf{Input}: Set of points  $T=\{ (x_i,y_i) \}_{i\in[n]}$ and $\eps>0$. 
        \State \textbf{Output}: A vector $\widehat v \in \R^d$.
        \State Find $\widehat{\sigma}_y$ such that $|\widehat{\sigma}_y - \sigma_y^2| = O(\sigma_y^2 \eps \log(1/\eps) )$.
        \State Draw $\alpha \in \R$ uniformly at random from $[-\widehat{\sigma}_y,\widehat{\sigma}_y]$.
        \State Define $I = [\alpha - \ell, \alpha +\ell]$ for $\ell:= \widehat{\sigma}_y/\log(1/\eps)$.
        \State $T' \gets \{x : (x,y) \in T, y \in I \}$.
        \State Let $\widehat{\beta}_I$ be the output of \Cref{alg:main} on $T'$.
        \State Return $\widehat{\beta} := (\widehat{\sigma}_y^2/\alpha)\widehat{\beta}_I$.
	\end{algorithmic}
\end{algorithm}

\begin{remark}
\label{rem:norm}
We further expand on the norm constraint of $\|\beta\|_2 = O(\sigma)$ in \Cref{thm:lin-regression}. The algorithm in \cite{LiuSLC20} obtains an error of $O(\sigma \eps \log(1/\eps))$ but their sample complexity scales multiplicatively with $\log(\|\beta\|_2/\eps\sigma)$.
If we do not assume a norm constraint on $\beta$, the sample complexity in \Cref{thm:lin-regression} would also have an extra multiplicative term of $\log(\|\beta\|_2/\eps\sigma)$ if we use \cite{LiuSLC20} as a warm-start. 
However, this factor of $\log(\|\beta\|_2/\eps\sigma)$ does not appear in the information-theoretical rate and can potentially be removed using either a tighter analysis of \cite{LiuSLC20} or a different (computationally-efficient) algorithm.
    
\end{remark}

\section{Discussion}

In this paper, we presented the first computationally-efficient algorithms that achieve the information-theoretic optimal error under Huber contamination for various sparse estimation tasks.
We now discuss some immediate open problems. 
Starting with mean estimation, our algorithm (\Cref{thm:main}) needs to know the covariance matrix of the inlier distribution; developing a sample and computationally-efficient algorithm for an unknown covariance matrix remains an important open problem.
More broadly, one could consider robust covariance estimation in the sparse operator norm or the $(\fr,k,k)$ norm.
For robust PCA (\Cref{thm:sparse-pca}), our algorithm works for a somewhat restricted range of the spike parameter $\rho$.
Removing this spike assumption and, more broadly, developing a sparse PCA algorithm for the gap-free setting (similar to \cite{JamLT20,KonSKO20,DiaKPP23-pca} in the dense setting)  remains open.

\bibliography{allrefs}
\bibliographystyle{alpha}

\newpage

\appendix

\section*{Appendix} 

The Appendix is structured as follows: \Cref{app:prelims} includes omitted preliminaries and \Cref{sec:appendix_pca} provides the proof of \Cref{thm:sparse-pca} for PCA.

\section{Preliminaries}\label{app:prelims}

This section contains additional preliminaries and omitted facts and proofs.

\paragraph{Additional Notation}
If $U \subseteq [d] \times [d]$ is a set of pairs of indices such that for every $(i,j) \in U$, $(j,i)$ is also in $U$, then for any matrix $\vec A \in \R^{d \times d}$ we denote by $(\vec A)_U$ the matrix restricted to the entries from $U$.
We use $x \lesssim y$ to denote that $x \leq C y$ for some absolute constant $C$. 
We use the notation $a \gg b$ to mean that $a > C b$ where $C$ is some sufficiently large constant.

In the next few subsections, we state some well-known facts  without proof, and some useful lemmata.

\subsection{Miscellaneous Facts}
\label{subsec:misc_facts}

\begin{fact}[Cover of the Sphere]\label{fact:cover}
  Let $r>0$. Let $B_R = \{ x \in \R^d : \|x \|_2 \leq  R \}$. There exists a set $\cC \subseteq B_R$
  such that $|\cC| \leq (1+2R/\eta)^d$ and for every $v \in B_R$
  we have that $\min_{y \in \cC} \|y-v\|_2 \leq \eta$.
\end{fact}

\begin{fact}\label{eq:ineq_matrices}
    For any square matrix $\vec A$ and positive-semidefinite matrix $\vec B$, 
    $\tr(\vec A \vec B) \leq \|\vec A\|_\op \tr(\vec B)$.
\end{fact}

\begin{fact}[see, for example, \cite{diakonikolas2023algorithmic}]
\label{fact:prunning}
    For $x,y \in \R^d$ with $y$ being a $k$-sparse vector, we have that $\|t_k(x)-y\|_2 \leq \sqrt{6} \|x-y\|_{2,k}$, where $\|\cdot\|_{2,k}$ denotes the sparse Euclidean norm (\Cref{def:sparsenorm}) and $t_k(x)$ the operator that sets all but the $k$ coordinates with the largest absolute values to zero.
\end{fact}

\begin{fact}\label{fact:fr-to-eucl}
    For unit vectors $w,v \in \R^d$, we have that $\|ww^\top - v v^\top \|_\fr = \Theta(\|w-v\|_2)$.
\end{fact}

\subsection{Probability Facts}

\begin{fact}[Gaussian Norm Concentration] \label{lem:norm-conc} 
    For every $0\leq \beta \leq \sigma \sqrt{d}$ we have that
    \begin{align*}
        \Pr_{X \sim \cN(\vec 0,\sigma^2 \vec I)}[|\|X\|_2- \sigma \sqrt{d}| > \beta] \leq
        2\exp\left(-\frac{\beta^2}{16\sigma^2}   \right) \;.
    \end{align*}
\end{fact}

\begin{fact}\label{fact:polynomial_var}
    For any $d \times d$ matrix $\vec A$, 
    it holds $\Var_{X \sim \cN(0,\vec I)}[X^\top \vec A X] = \|\vec A\|_\fr^2 + \tr(\vec A^2)$.
\end{fact}

\begin{definition}[Sub-Gaussian and Sub-gamma Random Variables]\label{def:orlich}
  A one-dimensional random variable $Y$ is sub-Gaussian if 
  $\| Y \|_{\psi_2} := \sup_{p\geq 1} p^{-1/2}\E\left[|Y|^p\right]$ is finite. 
  We say that $\| Y \|_{\psi_2}$ is the sub-Gaussian norm of $Y$. 
  A random vector $X$ in $\R^d$ is sub-Gaussian if 
  for every $v \in \cS^{d-1}$, $\| v^\top X \|_{\psi_2}$ is finite. 
  The sub-Gaussian norm of the vector is defined to be 
  \begin{align*}
    \| X \|_{\psi_2} := \sup_{v \in \cS^{d-1}} \| v^\top X \|_{\psi_2} \;.
  \end{align*}
  We call a centered one-dimensional random variable $Y$ a $(\nu,\alpha)_+$ sub-gamma if
  $\E[\exp(\lambda Y)] \leq \nu^2\lambda^2/2$ for all $0 \leq \lambda\leq 1/\alpha $.
  We call $\| Y \|_{\psi_1} := \sup_{p\geq 1}p^{-1} \E[|Y|^p]$ the sub-gamma norm of $Y$. 
\end{definition}

\begin{lemma}[Properties of Sub-gamma Random Variables~\cite{Wainwright19,BouLM13}]
\label{lem:sub-exp}
The class of sub-gamma random variables satisfy the following:
\begin{enumerate}
    \item \cite[Proposition 2.9]{Wainwright19} 
    If $Y$ is a centered $(\nu,\alpha)_+$ sub-gamma random variable, 
    then with probability $1-\delta$,
    $Y \lesssim \nu \sqrt{\log(1/\delta)} + \alpha \log(1/\delta)$.
    
    \item  \cite[Theorem 2.3]{BouLM13}   
    If $Y$ is a centered random variable satisfying that for all $\delta\in(0,1)$, 
    $Y \leq \nu \sqrt{\log(1/\delta)} + \alpha \log(1/\delta)$, 
    then $Y$ is $(\nu',\alpha')_+$ sub-gamma with 
    $\nu' \lesssim \nu + \alpha$ and $\alpha' \lesssim \alpha$.
    
    \item  \cite[Section 2.1.3]{Wainwright19}    
    Let $Y_1,\dots,Y_k$ be $k$ centered independent $(\nu,\alpha)_+$ sub-gamma random variables.
    Then $\sum_{i=1}^k Y_i $ is a $(\nu \sqrt{k},\alpha)_+$ sub-gamma random variable.
\end{enumerate}
\end{lemma}

\begin{fact}[Hanson-Wright Inequality]\label{fact:hw}
   For $X \sim \cN(0,  \vec I)$ in $\R^d$ and for every square $d \times d$ matrix $\vec A$ 
   and scalar $t\geq 0$, the following holds:
    \begin{align*}
        \Pr[|X^\top  \vec A X - \E[X^\top  \vec A X]| > t] \leq 
        2 \exp \left( -0.1  \min\left( \frac{t^2}{  \| \vec A \|_\fr^2},  \frac{t}{  \| \vec A \|_\op} \right)\right) \;.
    \end{align*}
\end{fact}

\subsection{Deterministic Conditions}
\label{subsec:deterministic}

The correctness of our algorithm will require the inliers to satisfy generic structural properties defined in \Cref{DefModGoodnessCond}. 

Recall that for a distribution $D$ and a weight function $w$, 
we denote the weighted (and appropriately normalized) version of $D$ using $w$ by $D_w$. 
We further use $\mu_{D_w}$ to denote the mean of $D_w$. 
We restate the conditions and then state formally the lemma showing that Gaussian samples satisfy them with high probability. 
\GOODNESS*

The following lemma demonstrates that 
the uniform distribution over a sufficiently large set of samples drawn from $\cN(\mu, \vec I_d)$
satisfies the deterministic conditions with respect to $\mu$. 

\begin{restatable}[Sample Complexity of Goodness Conditions]{lemma}{SAMPLES}\label{lem:samples}
    Let $\eps_0>0$ be a sufficiently small absolute constant. 
    Let $S$ be a set of $n$ samples drawn i.i.d. from $\cN(\mu,\vec I_d)$. 
    Let $G$ denote the uniform distribution on the points from $S$. 
    If $\eps<\eps_0$,  $k^2 \leq d$ and $n \gg  \frac{1}{\min\{\eps^2,\alpha^2\}} (k^2\log(d) + \log(1/\delta))\polylog(1/\eps)$, 
    then with probability at least $1-\delta$, $G$ is ($\eps,\alpha,k$)-good with respect to $\mu$.
\end{restatable}

\begin{proof}

    The proof of \Cref{cond:mean,cond:covariance} can be found in prior work (see, e.g., \cite{li2018principled}).
    \Cref{cond:linear_tails}
    uses that for all $k$-sparse unit vectors $v$, and $T \geq 6$, 
$\Pr_{X \sim G}[ |v^\top (X -\mu)| \geq T] \leq 3 \mathrm{erfc}\left( \frac{T}{\sqrt{2}} \right)  + \frac{\eps^2}{T^2} $, which can be found in Lemma 4.3 in \cite{DKKPS19-sparse}.

    We prove the remaining conditions below.
    
    \paragraph{Proof of \Cref{cond:polynomials}:} 
    We will use the notation 
    $g_{\vec A}(x) = (x-\mu)^\top\vec  A (x-\mu)$ (i.e., we do not include the centering $\tr(\vec A)$ in the polynomial and also we write the matrix $A$ in the subscript for extra clarity) and 
    $\tau_{\vec A}(x) = (g_{\vec A}(x) -\tr(\vec A)) \1(g_{\vec A}(x) -\tr(\vec A) > 100 \log(1/\eps))$.
    The proof will consist of the following steps (the last two steps involve a cover argument):
    
    \begin{enumerate}[label=(\arabic*)]
        \item First, we show that $\E_{X \sim \cN(\mu,I)}[\tau_{\vec A}(X)] \lesssim \eps^4$ 
        for every $\vec A$ of the form mentioned in \Cref{cond:polynomials}.\label{item:E[tau]_is_bounded}
    
        \item We then show that $\tau_{\vec A}(X) - \E_{X \sim \cN(\mu,\vec I)}[\tau_{\vec A}(X)]$ 
        for $X \sim \cN(\mu,\vec I)$ is a sub-gamma random variable for every $\vec A$ of the form
        mentioned in \Cref{cond:polynomials}.\label{item:tau_is_subexponential}
        
        \item Then, we show that for any fixed $\vec A$ of that form, 
        if $X_1,\ldots, X_n$ are i.i.d. samples from $\cN(\mu,\vec I)$, with probability $1-\delta$,
        we have 
        $   \frac{1}{n}\sum_{i=1}^n\tau(X_i) - \E_{X \sim \cN(\mu,\vec I)}[\tau(X)] 
    \lesssim \frac{1}{\sqrt{n}}\|\vec A\|_\fr \sqrt{\log(1/\delta)} + \frac{1}{n}\|\vec A\|_\op \log(1/\delta)$.\label{item:empirical_mean_bounded_for_single_A}
        
        \item Finally, with probability $1-\delta'$, $
    \E_{X \sim S}[\tau_{\vec A}(X)] 
    \lesssim \eps^4 + \sqrt{\frac{\log(1/\eps)(k^2\log(d)+\log(1/\delta'))}{n}} 
    +\frac{ k^2\log(d/\eps)+\log(1/\delta')}{n}$ 
        holds \emph{simultaneously} for all matrices $\vec A$ of the form mentioned in \Cref{cond:polynomials}.\label{item:empirical_mean_bounded_for_ALL_A}
    \end{enumerate}
    
    \vspace{11pt}
    
    We now prove the claims above.\\\\
    \emph{Proof of
    \Cref{item:E[tau]_is_bounded}:} 
    Using the Hanson-Wright inequality (\Cref{fact:hw}), we have that
    \begin{align}
        \Pr_{X \sim \cN(\mu,\vec I)}[|g_{\vec A}(X) - \tr(\vec A)  |> t] \leq 2 \exp\left( - 0.1 \min\left(\frac{t^2}{\|\vec A\|_\fr^2 }, \frac{t}{ \| \vec A \|_\op }  \right)  \right) \;. \label{eq:hw}
    \end{align}
    Setting $t=100\log(1/\eps)$, $\|\vec A\|_\fr \leq \sqrt{\log(1/\eps)}$, and $\|\vec A\|_\op \leq 1$, the above becomes $\Pr_{X \sim \cN(\mu,\vec I)}[ g_{\vec A}(X) - \tr(\vec A) > 100 \log(1/\eps)] \lesssim \eps^{10}$. This allows us to upper bound $\E_{X \sim \cN(\mu,\vec I) }[\tau_{\vec A}(X)]$ by $O(\eps^4)$ as follows: 
    \begin{align}
        \E_{X \sim \cN(\mu,\vec I) }[\tau_{\vec A}(X)] &= \E_{X \sim \cN(\mu,\vec I) }[(g_{\vec A}(X) - \tr(\vec A))\1(g_{\vec A}(X)-\tr(\vec A)>t)] \notag \\
        &\leq \sqrt{ \Pr_{X \sim \cN(\mu,\vec I) }[g_{\vec A}(X) -\tr(\vec A)>t] \E_{X \sim \cN(\mu,\vec I) }[(g_{\vec A}(X) - \tr(\vec A))^2]} \tag{by Cauchy-Schwarz} \\
        &= \sqrt{\Pr_{X \sim \cN(\mu,\vec I) }[g_{\vec A}(X) -\tr(\vec A)>t]} \sqrt{\Var_{X \sim \cN(\mu,\vec I)}[g_{\vec A}(X)]}   \notag \\
        &\lesssim \eps^5  \sqrt{\|\vec A\|_\fr+\tr(A^2)} \leq \eps^5 \sqrt{\|\vec A\|_\fr+\|\vec A\|_\fr^2} \lesssim \eps^5 \log(1/\eps) \lesssim \eps^4  \;, \label{eq:exp}
    \end{align}
    where the final inequality uses \Cref{fact:polynomial_var} and then it uses the fact that $\tr(\vec A^2) \leq \|\vec A \|_\fr^2$, which can be seen as follows: Let $A_i$ denote the rows of $\vec A$ and $\tilde{A}_i$ the rows, then $\tr(\vec A^2) = \sum_{i=1}^d A_i^\top \tilde{A}_i \leq  \sum_{i=1}^d \|A_i\|_2 \|\tilde{A}_i\|_2 \leq  \sum_{i=1}^d (\|A_i\|_2^2+ \|\tilde{A}_i\|_2^2)/2 \leq \|\vec A\|_\fr^2$.\\\\
\emph{Proof of \Cref{item:tau_is_subexponential}:}
Re-writing \Cref{eq:hw}, we see that
\begin{align} \label{eq:tmp1}
    \left| g_{\vec A}(X) - \tr(\vec A) \right| \lesssim \|\vec A\|_\fr \sqrt{\log(1/\delta)} + \|\vec A\|_\op \log(1/\delta)\;.
\end{align}
Moreover, we can write that 
\begin{align}
    \tau_{\vec A}(X) - \E_{X \sim \cN(\mu,\vec I)}[\tau_{\vec A}(X)]
    &\leq \tau_{\vec A}(X) \tag{$\tau_{\vec A}(X)\geq 0$ by definition} \\
    &= (g_{\vec A}(X) - \tr(\vec A))\1(g_{\vec A}(X)-\tr(\vec A)>100\log(1/\eps)) \notag \\
    &\leq | g_{\vec A}(X) - \tr(\vec A) | \notag \\
    &\lesssim \|\vec A\|_\fr \sqrt{\log(1/\delta)} + \|\vec A\|_\op \log(1/\delta) \tag{by \eqref{eq:tmp1}}
\end{align}
which means that the random variable $\tau_{\vec A}(x) - \E_{X \sim \cN(\mu,\vec I)}[\tau_{\vec A}(X)]$ is also $(\nu,\beta)_ +  $-sub-gamma with $\nu=\|\vec A\|_\fr$ and $\beta=\|\vec A\|_\op$.\\\\
\emph{Proof of \Cref{item:empirical_mean_bounded_for_single_A}} 
By \Cref{lem:sub-exp}, \Cref{item:tau_is_subexponential} implies that for i.i.d. samples 
$X_1,\ldots,X_n \sim \cN(\mu,\vec I)$, the average $\frac{1}{n}\sum_{i=1}^n\tau_{\vec A}(X_i) - \E_{X \sim \cN(\mu,\vec I)}[\tau_{\vec A}(X)]$ 
is ($\|\vec A\|_\fr/\sqrt{n},\|\vec A\|_\op/n$)-sub-gamma or, equivalently, with probability $1-\delta$ it holds:
\begin{align}
    \frac{1}{n}\sum_{i=1}^n\tau_{\vec A}(X_i) - \E_{X \sim \cN(\mu,\vec I)}[\tau_{\vec A}(X)] 
    \lesssim \frac{1}{\sqrt{n}}\|\vec A\|_\fr \sqrt{\log(1/\delta)} + \frac{1}{n}\|\vec A\|_\op \log(1/\delta) \;. \label{eq:fixed}
\end{align}
\emph{Proof of \Cref{item:empirical_mean_bounded_for_ALL_A}:}
The above is about a fixed choice of the matrix $\vec A$ from the set $\cV:=\{\vec A \in \R^{d \times d} : \|\vec A\|_\fr\leq \sqrt{\log(1/\eps)}, \|\vec A\|_\op=1, \text{$\vec A$ has at most $k^2$ non-zero elements} \}$. 
To show that concentration holds for all $\vec A$ in $\cV$, we take a cover set $\cV_\eta \subset \cV$. For every $\eta>0$ let $\cV_\eta$ be an $\eta$-cover of $\cV$, i.e., a set $\cV_\eta$ such that for every $\vec A \in \cV$ there exists an $\vec A' \in \cV_\eta$ with $\|\vec A-\vec A'\|_\fr \leq \eta$.
The next claim shows that the size of $V_\eta$ is not too large. 

\begin{claim}\label{cl:cover}
    The size of $\cV_\eta$ is upper bounded by $(6\sqrt{\log(1/\eps)}/\eta)^{k^2}\binom{d}{k^2}$. 
\end{claim}
\begin{proof}
    If we look at the flattened version of the matrix as a vector in $\R^{d \times d}$, 
    there are $\binom{d}{k^2}$-many ways to select which are the $k^2$ non-zero elements, 
    and for each such choice of the non-zero elements there exists a cover of size $(3\sqrt{\log(1/\eps)}/\eta)^{k^2}$ by \Cref{fact:cover}. 
    The union of all of these sets covers  has size at most $(3\sqrt{\log(1/\eps)}/\eta)^{k^2}\binom{d}{k^2}$ but 
    may not necessarily be a subset of $\cV$. 
    However, by Exercise 4.2.9 in \cite{Ver18} there exists a $\cV_\eta \subseteq \cV$ with size same as
    before but by replacing $\eta$ by $\eta/2$. 
\end{proof}

We will choose $\eta = 0.0001\eps^4/d^{4}$ (the reason for this choice will be clear later on). 
By setting $\delta = \delta'/(6\sqrt{\log(1/\eps)}/\eta)^{k^2}\binom{d}{k^2}$ and by a union bound, 
we can ensure that \eqref{eq:fixed} %
holds for all $A \in \cV_\eta$ \emph{simultaneously}. 
We will now show that, by continuity, 
the upper bound of \eqref{eq:fixed} with some additional error terms holds for all $A \in \cV$ simultaneously.

We will need the following notation: 
let $\vec A \in \cV$ be an arbitrary element of $\cV$ and 
$\vec A' \in \cV_\eta$ satisfy $\|\vec  A - \vec A' \|_\fr \leq \eta$. 
Denote by $S = \{X_1,\ldots, X_n \}$ the set of i.i.d. samples from $\cN(\mu,\vec I)$. 
Also denote by $\Delta g(X) := (g_{\vec A}(x)-\tr(\vec A)) - (g_{\vec A'}(x)-\tr(\vec A'))$ the difference between the two polynomials.
The goal is to upper bound $\E_{X \sim S}[\tau_{\vec A}(X)]$ which we do in steps as follows: 
First, we rewrite $(g_{\vec A}(X)-\tr(\vec A)) = (g_{\vec A'}(X)-\tr(\vec A')) + \Delta g(X)$, to get
\begin{align}
    \E_{X \sim S}[\tau_{\vec A}(X)]
    &= \E_{X \sim S}[(g_{\vec A}(X)-\tr(\vec A))\1(g_{\vec A}(X)-\tr(\vec A)> 100\log(1/\eps))] \notag\\
    &\leq \E_{X \sim S}[|\Delta g(X)|] + \E_{X \sim S}[(g_{\vec A'}(X)-\tr(\vec A'))\1(g_{\vec A'}(x)-\tr(\vec A')> 100\log(1/\eps) - \Delta g(X))] \;.  \label{eq:tmp2}
\end{align}
We will now bound each of the terms above. 
For the first one, we show that the following bound holds with probability $1-\delta$ over the random selection of $S$:
\begin{align}
    \E_{X \sim S}[|\Delta g(X)|] 
    &\leq \E_{X \sim S}[|g_{\vec A}(X) - g_{\vec A'}(X)| ] + |\tr(\vec A - \vec A')| \notag \\
    &=  \left|\left\langle  \vec A - \vec A', \frac{1}{n} \sum_{i=1}^n (X_i-\mu) (X_i-\mu)^\top \right\rangle \right|  + |\tr(\vec A - \vec A')| \notag \\
    &\leq \| \vec A - \vec A' \|_\fr   \frac{1}{n} \sum_{i=1}^n \|X_i-\mu\|_2^2  + d \|\vec A - \vec A'\|_\fr \label{eq:temp5234}\\
    &\lesssim \eta d \;, \label{eq:tmp3}
\end{align}
where in \eqref{eq:temp5234} we used that $\langle \vec B,\vec C \rangle \leq \|\vec B\|_\fr \|\vec C\|_\fr$ 
and $\|xx^\top\|_\fr = \|x\|_2^2$, 
and in \eqref{eq:tmp3} we used that  $\| \vec A - \vec A' \|_\fr \leq \eta$ 
and that 
$ \sum_{i=1}^n \|X_i-\mu\|_2^2 \leq d + O(\log(1/\delta)/n) = O(d)$ by
Gaussian norm concentration \Cref{lem:norm-conc} combined with the last part of \Cref{lem:sub-exp}.

We now move to the second term in the RHS of \eqref{eq:tmp2}.
\begin{align}
    \E_{X \sim S}&[(g_{\vec A'}(X)-\tr(\vec A'))\1(g_{\vec A'}(x)-\tr(\vec A')> 100\log(1/\eps) - \Delta g(X))] \notag\\
    &= \E_{X \sim S}[(g_{\vec A'}(X)-\tr(\vec A'))\1(g_{\vec A'}(x)-\tr(\vec A')> 100\log(1/\eps) - \Delta g(X)) \1(\Delta g(X) < \log(1/\eps))] \notag \\
    &\quad+  \E_{X \sim S}[(g_{\vec A'}(X)-\tr(\vec A')) \1(g_{\vec A'}(x)-\tr(\vec A')> 100\log(1/\eps) - \Delta g(X)) \1(\Delta g(X) >  \log(1/\eps))] \notag \\
    &\leq \E_{X \sim S}[(g_{\vec A'}(X)-\tr(\vec A'))\1(g_{\vec A'}(X)-\tr(\vec A')> 99\log(1/\eps)) ] \notag\\
    &\quad+ \E_{X \sim S}[|g_{\vec A'}(X)-\tr(\vec A')|  \1(\Delta g(X) >  \log(1/\eps))] \;. \label{eq:tmp4}
\end{align}
The first term is almost the same as $\E_{X \sim S}[\tau_A(X)]$,\footnote{The only difference is that the constant is 99 instead of 100 but this does not affect the conclusion.} which by \Cref{eq:exp} and \eqref{eq:fixed} we know that is upper bounded as follows: 
\begin{align}
    \E_{X \sim S}[(g_{\vec A'}(X)-\tr(\vec A'))\1(g_{\vec A'}(X)-\tr(\vec A')> 99\log(1/\eps)) ]
    &\lesssim \eps^4 + \frac{\|\vec A\|_\fr\sqrt{\log(1/\delta)}}{\sqrt{n}}  + \frac{\|\vec A\|_\op \log(1/\delta)}{n} \\
    &\lesssim \eps^4 + \frac{\sqrt{\log(1/\eps)\log(1/\delta)}}{\sqrt{n}}  + \frac{\log(1/\delta)}{n}, \label{eq:fixed_empirical}
\end{align}
where the last line used the assumptions $\|\vec A\|_\fr \leq \sqrt{\log(1/\eps)}$ and $\|\vec A\|_\op\leq 1$.

For the second term in \eqref{eq:tmp4} we claim that 
\begin{align}
    \E_{X \sim S}[|g_{\vec A'}(X)-\tr(A')|  \1(\Delta g_{\vec A}(X) >  \log(1/\eps))]=0 \label{eq:temp50}
\end{align}
whenever $\|x-\mu\|_2 \leq 10 \sqrt{d}$ for all $x \in S$  because of the indicator: This is because $\Delta g(x) = \langle \vec A-\vec A',(x - \mu)(x - \mu)^\top \rangle + \tr(\vec A-\vec A') \leq \| \vec A - \vec A'\|_\fr \|x - \mu\|_2^2 + \tr(\vec A-\vec A') \leq \eta 100 d + d \|\vec A-\vec A'\|_\fr \leq  101 \eta d$ and if $\eta < 0.0001\log(1/\eps)/d$, then the previous quantity is less than $\log(1/\eps)$. The event that for all $x \in S$, $\|x-\mu\|_2 \leq 10 \sqrt{d}$ happens with overwhelming probability: $\Pr_{X_1,\ldots,X_i \sim \cN(\mu,I)}[\exists X_i \in S: \|X_i - \mu\|_2 > 10\sqrt{d}] \leq ne^{-d/100}$ by Gaussian norm concentration (\Cref{lem:norm-conc}) and a union bound.

Combining \eqref{eq:tmp2},\eqref{eq:tmp3},\eqref{eq:tmp4},\eqref{eq:fixed},\eqref{eq:temp50}, and choosing $\delta = \delta'/(6\sqrt{\log(1/\eps)}/\eta)^{k^2}\binom{d}{k^2}$, and $\eta = 0.0001\eps^4/d^{4}$,  we conclude that with probability $1-\delta'-ne^{-d/100}$,  the following holds simultaneously for all matrices $A$ in the cover:
\begin{align}
    \E_{X \sim S}[\tau_A(X)] &\lesssim \eps^4 + \sqrt{\frac{\log(1/\eps)\log(1/\delta)}{n}} + \frac{\log(1/\delta)}{n}   + d \eta  \notag \\
    &\lesssim \eps^4 + \sqrt{\frac{\log(1/\eps)\log(1/\delta)}{n}} + \frac{\log(1/\delta)}{n}  \;. \label{eq:finalbound}
\end{align}
Using $n \gg \eps^{-2} \polylog(1/\eps)(k^2\log(d/\eps) + \log(1/\delta'))$ 
we can make the RHS less than $\eps$. Finally, in order to have $ne^{-d/100} \leq \delta'$ we need $d \gg \polylog(1/(\eps \delta'))$ but we can assume that this is true without loss of generality by padding the samples with additional Gaussian coordinates (if the goodness conditions hold for the padded data, they continue to hold for the original data).

\item
\paragraph{Proof of \Cref{cond:extra} in \Cref{cond:polynomials}:} 
Qualitatively, the proof goes through the same arguments as the one for \Cref{cond:poly-conc}, thus we will not be that detailed and instead we will focus mostly on the few differences.
Denote $g_{\vec A}(x) = (x-\mu)^\top \vec A (x-\mu)$, $h_{\beta,v} = \beta + v^\top (x-\mu)$, and $\tau_{\vec A,\beta,v}(x) = (g_{\vec A}(x)-\tr(\vec A))\1(h(x) > 100 \log(1/\eps))$. For $t>1$, by Gaussian concentration $\Pr_{X \sim \cN(\mu,\vec I)}[h(x) > t] \leq e^{-\Omega(t^2)}$. Thus, for $t = 100 \log(1/\eps)$
\begin{align*}
    \E_{X \sim \cN(\mu,I)}[\tau_{\vec A}(X)] \leq \sqrt{\Pr_{X \sim \cN(\mu,\vec I)}[h(x) > t]}\sqrt{\Var_{X \sim \cN(\mu,\vec I)}[g_{\vec A}(X)]}
    \leq \eps^5 \|\vec A\|_\fr \lesssim \eps^4 \;.
\end{align*}
Similarly to \Cref{eq:tmp1}, $\tau_{\vec A,\beta,v}(X) - \E_{X \sim \cN(\mu,\vec I)}[\tau_{\vec A,\beta,v}(X)]$ is $(\|\vec A\|_\fr,\|\vec A\|_\op)$-sub-gamma. Therefore, the average of $n$ i.i.d. samples is $(\|\vec A\|_\fr/\sqrt{n},\|\vec A\|_\op/n)$-sub-gamma. We now let $\cV_\eta$ be the cover set of \Cref{cl:cover}, $\cC_\eta$ be the cover set of the $k$-sparse unit ball, which has size at most $(3/\eta)^k\binom{d}{k}$, and finally let $\cV'_\eta = \cV_\eta \times \cC_\eta$ be the product of the two. We choose $\eta = 0.0001 (\eps/d)^{10}$ and probability of failure $\delta = 1/|\cV'_\eta|$. By a union bound, we have that 
\begin{align}
    \frac{1}{n}\sum_{i=1}^n\tau_{\vec A',\beta',v'}(X_i) - \E_{X \sim \cN(\mu,\vec I)}[\tau_{\vec A',\beta',v'}(X)] 
    \lesssim \frac{1}{\sqrt{n}}\|\vec A'\|_\fr \sqrt{\log(1/\delta)} + \frac{1}{n}\|\vec A'\|_\op \log(1/\delta)  \label{eq:fixed_2}
\end{align}
holds simultaneously for all $\vec A',\beta',v'$ inside the cover set. Now let arbitrary $\vec A,\beta,v$. We can show that \eqref{eq:fixed_2} will still hold, with some additional error terms. First, let $\Delta g(x) = (g_{\vec A}(x)-\tr(\vec A))-(g_{\vec A'}(x)-\tr(\vec A')$ and $\Delta h(x) = \beta-\beta' + (v-v')^\top(x-\mu)$. Also, denote by $S=\{X_1,\ldots,X_n\}$ the set of $n$ samples. We can write 
\begin{align*}
    \E_{X \sim S}[\tau_{\vec A,\beta,v}(X)]
    \leq \E_{X \sim S}[|\Delta g(X)|] + \E_{X \sim S}[(g_{\vec A'}(X)-\tr(\vec A'))\1(h'(x) > 100 \log(1/\eps) - \Delta h(X))] \;.
\end{align*}
The first term is at most $O(\eta d)$ as in \Cref{eq:tmp3}. We bound the second term as follows:
\begin{align}
    \E_{X \sim S}&[(g_{\vec A'}(X)-\tr(\vec A'))\1(h'(x) > 100 \log(1/\eps) - \Delta h(X))]\\
    &\leq \E_{X \sim S}[(g_{\vec A'}(X)-\tr(\vec A'))\1(h'(x) > 99 \log(1/\eps))]
    + \E_{X \sim S}[|g_{\vec A'}(X)-\tr(\vec A')|\1( \Delta h(X) >  \log(1/\eps) )] \label{eq:secondterm}
\end{align}
The first term above is bounded as in \Cref{eq:fixed_2} (the only change is that the constant 100 is now 99 but that should only affect the constant in the RHS of \Cref{eq:fixed_2}). For the second term, we note that with probability $1-ne^{-d/100}$ we have $\|x - \mu\|_2 \leq 10 \sqrt{d}$ for all $x \in S$. Since $\Delta h(x) = \beta-\beta' + (v-v')^\top(x-\mu)$ and we have designed the cover such that $ |\beta-\beta'|$ and $\|v-v'\|_2$ are at most $\eta \leq 0.0001(\eps/d)^{10}$, we have that $\1( \Delta h(X) >  \log(1/\eps) )=0$ for all $X \in S$ under that event. Thus, with probability $1-ne^{-d/100}$ over the dataset $S$, the second term in \Cref{eq:secondterm} is zero. Putting everything together, we have the same bound as in \Cref{eq:finalbound}.

\item
\paragraph{Proof of \Cref{cond:hw} in \Cref{cond:polynomials}:} 
Using \eqref{eq:hw} with $t=10\log(1/\eps)$, $\|\vec A\|_\fr \leq \sqrt{\log(1/\eps)}$ and $\|\vec A\|_\op=1$, 
we obtain that $\Pr_{X \sim \cN(\mu,\vec I)}[p(X) > 20 \log(1/\eps)] \lesssim \eps^2$. 
Now by a basic application of Chernoff bounds 
we have that $|\Pr_{X \sim S}[p(X) > 20 \log(1/\eps)]-\Pr_{X \sim \cN(\mu,\vec I)}[p(X) > 20 \log(1/\eps)]| \leq \eta$ with probability $1-e^{-\Omega(\eta^2 n)}$, 
where we will use $\eta = \eps$. 
Thus, if $n\gg \eps^{-2} \log(1/\delta)$, 
we have that $\Pr_{X \sim S}[p(X) > 20 \log(1/\eps)] \lesssim \eps$ 
with probability at least $1-\delta$.

\end{proof}

\subsection{Certificate Lemma}
We restate and prove the following lemma.
\CERTIFICATE*

\begin{proof}
Let $\rho = \eps \E_{X \sim B}[w(X)]/ \E_{X \sim P}[w(X)]$. 
Recall our notation
$P_w(x) = w(x)P(x)/\E_{X \sim P}[w(X)]$, 
$B_w(x) = w(x)B(x)/\E_{X \sim B}[w(X)], G_w(x) = w(x)G(x)/\E_{X \sim G}[w(X)]$ 
for the weighted versions of the distributions $P,B,G$ 
and denote the corresponding covariance matrices by $\vec \Sigma_{P_w}, \vec \Sigma_{B_w},  \vec \Sigma_{G_w}$. 
We can write:
\begin{align}
  \vec \Sigma_{P_w} &= \rho \vec  \Sigma_{B_w} + (1 - \rho) \vec  \Sigma_{G_w} + \rho(1 - \rho)(\mu_{G_w} - \mu_{B_w})(\mu_{G_w} - \mu_{B_w})^\top \;. \label{eq:decomposision_of_cov}
\end{align}
Let $v$ be a $k$-sparse unit norm vector. 
Since $v^\top \vec \Sigma_{P_w} v \leq 1 + \lambda$, we obtain the following:
\begin{align*}
1+\lambda &\geq v^\top \vec \Sigma_{P_w} v 
\geq (1 - \rho)   v^\top \vec \Sigma_{G_w} v + \rho(1 - \rho)   (v^\top(\mu_{B_w} - \mu_{G_w}))^2 \\
&\geq (1 - \rho)\left( 1 - \alpha \log(1/\alpha) \right) + \rho(1 - \rho) (v^\top(\mu_{B_w} - \mu_{G_w}))^2_2 \;,
\end{align*}
where the second step uses \Cref{eq:decomposision_of_cov} and the last step uses the fact that $G$ satisfies  \Cref{cond:covariance} of \Cref{DefModGoodnessCond}.
The expression above implies the following:
\begin{align*}
(v^\top(\mu_{B_w} - \mu_{G_w}))^2 \leq   \frac{\lambda + \rho +  \alpha \log(1/\alpha)}{\rho(1 - \rho)} \;.
\end{align*}
We can now bound the error $| v^\top( \mu_{P_w} - \mu)|$ as follows:
\begin{align*}
| v^\top( \mu_{P_w} - \mu)| &= | v^\top(\mu_{G_w} - \mu) + \rho  v^\top (\mu_{B_w} - \mu_{G_w}) | \\
&\leq |v^\top( \mu_{G_w} - \mu)| + \rho | v^\top(\mu_{B_w} - \mu_{G_w})| \\
&\leq \|\mu_{G_w} - \mu\|_{2,k} + \rho |v^\top (\mu_{B_w} - \mu_{G_w} )|\\
&\leq \alpha \sqrt{\log(1/\alpha)} + \sqrt{  \rho}\sqrt{\frac{\lambda + \rho + \alpha \log(1/\alpha)}{1 - \rho}},
\end{align*}
where the last inequality uses that $G$ satisfies \Cref{cond:mean}.
We now use bounds on $\rho$ to simplify the terms.
Recall that $\rho= \eps \E_{X \sim B}[w(X)]/ \E_{X \sim P}[w(X)] \leq  \epsilon/(1 - \alpha)$. As $\alpha< 1/2$, we get that $\rho < 2 \epsilon $. In addition, note that $ \rho < 1/2$. Using these, we conclude that
\begin{align*}
\| \mu_{P_w} - \mu \|_{2,k} &\leq | v^\top( \mu_{P_w} - \mu)| \lesssim \alpha \sqrt{\log(1/\alpha)} + \sqrt{  \lambda \epsilon} +  \epsilon +   \sqrt{\alpha \eps \log(1/\alpha)}
\;.
\end{align*}
   
\end{proof}

\subsection{Useful Procedures from Robust Statistics}

The following result is implicit in \cite{DKKPS19-sparse}, which gives algorithm for $k$-sparse mean estimation with sub-optimal error of $O(\eps\sqrt{\log(1/\eps))}$. Since that algorithm relies on a certificate lemma similar to \Cref{LemCertiAlt_restated}, it filters out points until the $k$-sparse norm of the empirical covariance minus identity becomes $O(\eps \polylog(1/\eps))$. Even though it does not manage to make the variance as small as would be required for the Certificate Lemma to yield $O(\eps)$ error, this is a useful starting point, and we will use it as a preprocessing of the data in order to assume an $O(\eps \polylog(1/\eps))$ bound on the variance throughout our proofs.

\begin{fact}\label{fact:preprocess}
    Let $D \sim \cN(\mu,\vec I)$ be a Gaussian distribution on $\R^d$ with unknown mean $\mu$ 
    and $\eps \leq \eps_0$ for some sufficiently small absolute constant $\eps_0>0$. 
    Let $T$ be a set of $n$ samples from an $\eps$-corrupted version of $D$, according to the Huber contamination model.
    If $n \gg (k^2\log(d) +\log(1/\delta))\polylog(1/\eps)/\eps^2$, 
    the algorithm from \cite{DKKPS19-sparse} run on input $T,k,\eps$ 
    outputs a subset $T' \subseteq T$ of the original dataset 
    such that the following hold with probability at least $1-\delta$:
    \begin{enumerate}
        \item (Algorithm deletes $\log(1/\eps)$ more outliers than inliers) 
        If $S$ denotes the set of inliers in $T$, 
        we have that $|(T \setminus T') \cap S| \leq \frac{1}{\log(1/\eps)}|(T \setminus T') \cap (T \setminus S)|$.
        \item (Small ($F,k,k$)-norm) 
        Denoting by $\vec \Sigma_{T'}$ the covariance matrix of the output set $T'$, 
        we have that $\|\vec \Sigma_{T'} - \vec I\|_{F,k,k} = O(\eps \log^2(1/\eps))$.
        \item (Estimate of true mean) The empirical mean $\mu_{T'}$ of the output dataset satisfies $\| \mu_{T'} - \mu \|_{2,k} \lesssim \eps \log(1/\eps)$.
    \end{enumerate}
\end{fact}
\begin{proof}[Proof sketch]
    \emph{(Algorithm deletes $\log(1/\eps)$ more outliers than inliers):}
    Algorithm 1 from \cite{DKKPS19-sparse} filters points while ensuring 
    that for every inlier deleted, $\Omega(1)$ outliers are deleted. 
    This is done by eliminating points according to the tail probabilities
    of the data. If the tail beyond some point $T$ has a constant fraction more
    mass than it should, then a constant fraction of the points that are
    eliminated should be outliers. 
    To boost the ratio of inliers to outliers being deleted, 
    one just needs to adjust the threshold in lines 7 and 10 in Algorithm $1$ of \cite{DKKPS19-sparse} so that the mass beyond $T$
    has to be more than $1+\log(1/\eps)$ times the mass of the inliers. 
    This will imply that the ratio of inliers to outliers deleted can be is boosted to $\log(1/\eps)$. 
    The cost that we need to pay for this change is that 
    the stopping condition in Line 4 of Algorithm 1 in \cite{DKKPS19-sparse}
    changes from 
    $\|(\vec \Sigma_{T'} - \vec I)_{U}\|_\fr \leq O(\eps \log(1/\eps))$ 
    to  $\|(\vec \Sigma_{T'} - \vec I)_{U}\|_\fr \leq O(\eps \log^2(1/\eps))$.    

    \emph{(Small (F,$k,k$)-norm of output covariance matrix):} 
    Running the (modified version of) Algorithm 1 
    in \cite{DKKPS19-sparse} with sparsity set to $2k$ (instead of $k$) ensures that upon termination we have that 
    $\|(\vec \Sigma_{T'} - \vec I)_{U}\|_\fr = O(\eps \log^2(1/\eps))$ for $U \subset [d] \times [d]$ being the set of the $2k$ largest magnitude diagonal entries of $\vec \Sigma_{T'} - \vec I$ and the largest magnitude $4k^2 - 2k$ off diagonal entries, 
    with ties broken so that if $(i,j) \in U$ then $(j,i) \in U$. 
    By using the definition of $(F,k,k)$-norm, this implies that:
    \begin{align} \label{eq:fkkbound}
        \| \vec \Sigma_{T'} - \vec I \|_{F,k,k} \leq \|(\vec \Sigma_{T'} - \vec I)_{U}\|_\fr \leq O(\eps \log^2(1/\eps)) \;.
    \end{align}

    (Estimate of the mean): Since the algorithm deleted $\log(1/\eps)$ factor more outliers than inliers and initially we have $\eps n$ outliers, it must be the case  that the inliers after filtering are at least $ (1- \eps/\log(1/\eps))n$. Also note that $\sup_{v\in \R^d: \|v\|_2=1, \|v\|_0 = k}v^\top (\vec \Sigma_{T'} - \vec I) v = O(\eps \log^2(1/\eps))$, which comes from \eqref{eq:fkkbound} and \Cref{fact:h_1_lower_bound}. This allows us to use \Cref{LemCertiAlt_restated}, which implies that $\| \mu_{T'} - \mu \|_{2,k} \lesssim \eps \log(1/\eps)$ (for this application we also used that the inliers in the dataset satisfy the goodness conditions of \Cref{DefModGoodnessCond} with $\alpha = \eps/\log(1/\eps))$). 
    
\end{proof}

\section{Robust Principal Component Analysis}\label{sec:appendix_pca}

In this section we complete the proof of \Cref{thm:sparse-pca} by providing the missing proofs of the claims mentioned in \Cref{sec:pca_main_body}.

We state and prove a more detailed version of \Cref{claim:conditional} below:
\begin{restatable}{claim}{CONDITIONAL2}\label{claim:conditional2}
    Let $X \sim \cN(0,\vec I+\rho v v^\top)$ be a random variable from the spiked covariance model and $Z = \mathrm{Proj}_{w^\top}(X)$ the projection of $X$ onto the subspace perpendicular to $w$. For $\alpha \in \R$, let $G_\alpha$ denote the distribution of $Z$ conditioned on $w^\top X = \alpha$, and for an interval $I \subset \R$, let $G_{I}$ denote the distribution of $Z$ conditioned on $w^\top X \in I$. We also use the notation $\phi(z; \mu,\vec \Sigma)$ for the pdf of $\cN(\mu,\vec \Sigma)$. Then, the pdfs of the two aforementioned distributions are:
    \begin{enumerate}
        \item $G_\alpha(z) = \phi(z; \tilde{\mu}, \tilde{\vec \Sigma})$ with  $\tilde{\mu}=\frac{\rho (w^\top v) \alpha }{1+\rho (w^\top v)^2} \bar{v}$ and $ \tilde{\vec \Sigma} = \vec I + \frac{\rho }{1+ \rho (w^\top v)^2}\bar{v}\bar{v}^\top$, where $\bar{v}:= (v-(w^\top v)w)$.
        \item $G_I(z) = \frac{1}{\Pr_{X \sim \cN(0,\vec I+\rho v v^\top)}[ w^\top X \in I]} \int_{\alpha' \in I} G_{\alpha'}(z) \Pr_{X \sim \cN(0,I+\rho v v^\top)}[w^\top X = \alpha'] \d \alpha'$.
    \end{enumerate}
\end{restatable}
\begin{proof}
    Let $e_1,\ldots,e_d$ denote the standard orthonormal basis of $\R^d$.
    By a rotation of the space, we can assume without loss of generality that $w$ is aligned with $e_d$, i.e., $w= \|w\|_2 e_d$.  We denote by $\overline{v} = (v_1,\ldots,v_{d-1})$ the projection of $v$ to the subspace orthogonal to $w$ and by $v_d = w^\top v$ the projection of $v$ in the direction of $w$.  This way the subspace perpendicular to $w$ is the one spanned by the first $d-1$ basis elements.
    The first part of the claim follows from \Cref{fact:conditional} below applied with $y_1 = (x_1,\ldots,x_{d-1}), y_2= \alpha, \mu_{1} = 0, \mu_2=0, \vec \Sigma_{11} = \vec I+ \rho \bar{v} \bar{v}^\top, \vec \Sigma_{12} = \rho v_d \bar{v},\vec \Sigma_{21} = \rho v_d\bar{v}^\top$ and $\vec \Sigma_{22} = 1+ \rho v_d^2$ (and the fact that $v_d = w^\top v$ and $\overline{v} = v-(w^\top v)w$). The second claim follows by the law of total probability.

    \begin{fact} \label{fact:conditional}
    If $\left[\begin{matrix} y_1 \\ y_2 \end{matrix} \right] \sim \cN\left(\left[\begin{matrix} \mu_1 \\ \mu_2 \end{matrix} \right] , \left[ \begin{matrix}   \vec \Sigma_{11} &  \vec \Sigma_{12} \\  \vec  \Sigma_{21} &  \vec \Sigma_{22} \end{matrix} \right] \right)$, then $y_1 | y_2 \sim \cN(\bar{\mu}, \bar{  \vec \Sigma})$, with $\bar{\mu} = \mu_1 +  \vec  \Sigma_{12}  \vec \Sigma_{22}^{-1}(y_2 - \mu_2)$ and $ \bar{\vec \Sigma} = \vec \Sigma_{11} -   \vec \Sigma_{12}   \vec \Sigma_{22}^{-1}   \vec \Sigma_{21}$.
  \end{fact}
\end{proof}

We now proceed to show how to achieve the estimation guarantees in lines \ref{line:var_est} and \ref{line:cond_est} of \Cref{alg:pca}.

\begin{claim}\label{cl:applicability}
    There exists a computationally efficient estimator that uses $O(1/\eps)$ $\eps$-corrupted samples and achieves the guarantee in line \ref{line:var_est} of \Cref{alg:pca}. 
\end{claim}
\begin{proof}
    The estimator exists since it is known that robustly estimating the variance of a Gaussian in one dimension can be done with $O(\eps)$ error (see, e.g., \cite{DKKLMS18-soda} which is for high-dimensions but here we only need it for one dimension). Concretely, consider $(w^\top x)$ for $x \sim \cN(0,\vec I + \rho v v^\top)$. Then $(w^\top x) \sim \cN(0,1+ \rho(w^\top v)^2)$. By robustly estimating its variance, we can obtain $y'$ such that $|y' - (1 + \rho(w^\top v)^2) | = O(\eps(1+\rho))=O(\eps)$; here we used that $\rho=O(1)$ by assumption. Then $y:= (y' - 1)/\rho$ satisfies $|y - (w^\top v)^2 | = O(\eps/\rho)$, which is the guarantee mentioned in  line \ref{line:var_est} of \Cref{alg:pca}. 
\end{proof}

\begin{claim}\label{cl:applicability}
    There exists a computationally efficient estimator that uses $O((k^2\log(d)+ \polylog(1/\eps))/\eps^2)$ $\eps$-corrupted samples from $\cN(0,\vec I+\rho v v^\top)$ and achieves the guarantee in line \ref{line:cond_est} of \Cref{alg:pca}.
\end{claim}
\begin{proof}[Proof sketch]

    The estimator for line \ref{line:cond_est} of \Cref{alg:pca} is the following: Let $T$ be a set of $\eps$-corrupted samples from $\cN(0,\vec I+ \rho v v^\top)$ in the Huber contamination model.
    \begin{enumerate}
        \item Draw $\alpha$ uniformly from $[-(1+\rho),1+\rho] \setminus [-0.1,0.1]$.\footnote{We draw $\alpha$ from $[-(1+\rho),1+\rho] \setminus [-0.1,0.1]$ because we need $|\alpha|=\Omega(1)$ in \Cref{cl:closeness}.}
        \item Define the interval $I = [\alpha-\ell,\alpha+\ell]$ for $\ell = 1/\log(1/\eps)$.
        \item $T' = \{x \in T : w^\top x \in I \}$.
        \item $T'' = \{\mathrm{Proj}_{w^\perp}(x)  : x \in T' \}$ (Project samples to subspace orthogonal to $w$).
        \item Let $z$ be the output of \Cref{alg:main} run on $T''$.
        \item Output the vector obtained from $z$ after zeroing out all coordinates except the $2k$ ones with the largest absolute value.
    \end{enumerate}

    The correctness of   \Cref{alg:main} relies on the following two:  (i) the fraction of outliers in $T'$ continues to be $O(\eps)$ and (i.e., we are still in the Huber contamination model with approximately the same corruption level) (ii) the inliers satisfy the deterministic conditions from \Cref{DefModGoodnessCond}.

    We start with (i), where we sketch the proof. 
    We argue that the probability of an outlier $x$ having $w^\top x \in I$ divided by the probability that an inlier having $w^\top x \in I$ is at most $O(1)$. This would ensure that the fraction of outliers does not blow up by more than a constant factor. We start with upper bounding the probability for outliers. Since $I$ is chosen randomly and independently of anything else we can imagine that the outlier $x$ is fixed and then $I$ is chosen randomly.  Let us only examine the case where $w^\top x \in [-2(1+\rho),2(1+\rho)]$ (because otherwise  $w^\top x \not\in I$). The probability that $w^\top x \in I$ is then the ratio of the length of $I$ to the length of the interval $[-2(1+\rho),2(1+\rho)]$, which is $O(\ell/(1+\rho))$. We now argue about the inliers. For this, we can imagine that $I$ is fixed and the inlier $x$ is drawn randomly from the inlier distribution. We note that $w^\top x \sim \cN(0,\tilde{\sigma}^2)$ with $\tilde{\sigma}^2 := 1+\rho (w^\top v)^2$. Since $I \subseteq [-3 \tilde{\sigma}^2, 3\tilde{\sigma}^2]$ the Gaussian distribution behaves approximately uniformly there and thus the probability that $w^\top X\in I$ is $\Omega(\ell/\tilde{\sigma}^2)$, which is also $\Omega(\ell/(1+\rho))$.

    We next show (ii). The inliers follow the distribution $G_I$ from \Cref{claim:conditional}.
    We want to show that $G_I$ satisfies the deterministic conditions of \Cref{DefModGoodnessCond} with respect to $\mu_{G_\alpha}$, the mean of the distribution $X \sim \cN(0,\vec I+\rho v v^\top)$ conditioned on $w^\top X=\alpha$, where $\alpha$ is the center of the interval $I$.
    We sketch the argument for why each of the conditions holds.
    \Cref{cond:mean} and \Cref{cond:covariance} rely only on three properties: (i) $\| \mu_{G_I} - \mu_{\alpha} \|_2 = O(\eps)$,  (ii) $\| \vec I - \vec \Sigma_{G_I} \|_\op = O(\eps)$ and (iii) $O(1)$-subgaussianity. We can check that these hold. 
    For the first one,  we have that
    \begin{align*}
        \| \mu_{G_{\alpha'}} - \mu_{G_\alpha} \|_2 \lesssim \ell \frac{\rho  (w^\top v)}{1+\rho (w^\top v)^2} \|v - w  (w^\top v) \|_2 
        \lesssim \ell \rho (\eps\sqrt{\log(1/\eps)}/\rho) \lesssim \eps\;, \tag{using $\ell = 1/\log(1/\eps)$}
    \end{align*}
    where the last line uses  that $\alpha=O(1)$, $(w^\top v)^2 \leq 1$ and  $\|v - w  (w^\top v) \|_2 = O(\eps\sqrt{\log(1/\eps)}/\rho)$ by the guarantee of the estimator in line \ref{line:warmstart}.
    We move to the covariance property. Let $\alpha_1$ and $\alpha_2$ denote the bounds of the interval $I$, i.e., $I=[\alpha_1,\alpha_2]$. We have that $\E_{X \sim G_{\alpha'}}[XX^\top] = \vec I + \frac{\rho   }{1 + \rho (w^\top v)^2 } \bar{v} \bar{v}^\top + \frac{\rho^2(\alpha')^2}{(1+ \rho (w^\top v)^2) }\bar{v} \bar{v}^\top$ where $\bar{v} := v - (w^\top v) w$. Thus, as a convex combination, 
    $\E_{X \sim G_{I}}[XX^\top] = \vec I + \frac{\rho  }{1 + \rho (w^\top v)^2 } \bar{v} \bar{v}^\top + \frac{\xi \rho^2(\alpha_1)^2 + (1-\xi) \rho^2(\alpha_2)^2}{(1+ \rho (w^\top v)^2) }\bar{v} \bar{v}^\top$ for some $\xi \in [0,1]$.
    \begin{align*}
        \left\| \vec \Sigma_{G_I} - \vec I \right\|_{\op} &\leq 
        \left\|  \frac{\rho  }{1 + \rho (w^\top v)^2 } \bar{v} \bar{v}^\top + \frac{\xi \rho^2(\alpha_1)^2 + (1-\xi) \rho^2(\alpha_2)^2}{(1+ \rho (w^\top v)^2) }\bar{v} \bar{v}^\top \right\|_{\op} \\
        &\lesssim \rho \| \bar{v} \|_2^2 + \rho^2 \max(\alpha_1^2, \alpha_2^2) \| \bar{v} \|_2^2 \\
        &\lesssim \rho(\eps\sqrt{\log(1/\eps)}/\rho)^2 + \rho^2(\eps\sqrt{\log(1/\eps)}/\rho)^2\\
        &\lesssim \eps^2 \log(1/\eps)/\rho + (\eps \sqrt{\log(1/\eps)})^2 \lesssim \eps \;.  \tag{using $\rho \geq \eps \log(1/\eps)$}
    \end{align*}
    For the subgaussianity, we have that
    \begin{align*}
        \E_{X \sim G_I}[|v^\top(X - \mu_{G_I})|^p]^{1/p} &\leq \max_{\alpha' \in I}\E_{X \sim G_{\alpha'}}[|v^\top(X - \mu_{G_I})|^p]^{1/p}\\
         &\leq \max_{\alpha' \in I}\E_{X \sim G_{\alpha'}}[|v^\top(X - \mu_{G_{\alpha'}})|^p]^{1/p} + \| \mu_{G_{\alpha'}} - G_I \|_2 \\
        &\lesssim \sqrt{p} + \ell \frac{\rho \alpha (w^\top v)}{1+\rho (w^\top v)^2} \|v - w  (w^\top v) \|_2\\ 
         &\lesssim \sqrt{p} + \ell O(\eps\sqrt{\log(1/\eps)}) = \sqrt{p} +  O(1)\;.
    \end{align*}

    We now move to the remaining deterministic conditions about thresholded polynomials. Let us use the notation $p_{a}(x) = (x-\mu_\alpha)^\top \vec A (x-\mu_\alpha) - \tr(\vec A)$. Our goal is to show that $\E_{X \sim S}[ p_{\alpha}(x)   \1(p_{\alpha}(x) > 100 \log(1/\eps) )] \leq \eps$, where $S$ is a set of $n$ i.i.d. samples from $G_I$. As it can be seen by the proof \Cref{lem:samples}, the key element for proving that condition is the concentration in \eqref{eq:hw}. Thus it suffices to ensure that it holds for samples from $G_I$. Since $G_I$ is a mixture of the $G_{\alpha'}$ distributions, it suffices to show that the concentration holds for $G_{\alpha'}$ for all $\alpha' \in I$. Let $\Delta p(x) = p_{\alpha}(x) - p_{\alpha'}(x) = (\mu_{\alpha'} - \mu_{\alpha})^\top \vec A(\mu_{\alpha'} - \mu_{\alpha}) + 2 (x-\mu_{\alpha'})^\top \vec A (\mu_{\alpha'}-\mu_{\alpha}) $ be the difference of the two polynomials. By considering the cases where $|\Delta p(x)| \leq 50 \log(1/\eps))$ and $|\Delta p(x)| > 50 \log(1/\eps))$, we have that
    \begin{align}
        \Pr_{X \sim G_{\alpha'}} [| p_{\alpha}(X)  | \geq 100 \log(1/\eps)] 
        \leq \Pr_{X \sim G_{\alpha'}} [| p_{\alpha'}(X) | \geq 50 \log(1/\eps)]  +  \Pr_{X \sim G_{\alpha'}}[|\Delta p(X) | > 50\log(1/\eps)] \;. \label{eq:meanshift}
    \end{align}
    The first term is bounded using \Cref{cond:poly-conc} of \Cref{lem:samples} (although the lemma has been proved for Gaussians with identity covariance, it can be checked that it goes through for $X \sim G_{\alpha'}$ which is Gaussian with covariance $I + O(\eps)$). The second is bounded using Gaussian concentration ($\Delta p(X)$ is a linear polynomial).

    The previous discussion means that our mean estimator (\Cref{alg:main}) is applicable and satisfies the same guarantee as in \Cref{thm:main}, i.e., yields $z$ with $\|z-\mu_{G_I}\|_{2,k} = O(\eps)$ (which since $\|\mu_{G_I}-\mu_{\alpha}\|_2=O(\eps)$ from earlier also means that $\|z-\mu_{\alpha}\|_{2,k} = O(\eps)$). However the guarantee in line \ref{line:cond_est} of \Cref{alg:pca} that we are trying to prove uses $\ell_2$-norm instead of the $(2,k)$-norm. We explain below how one can get from the one norm bound to the other:
    We start by focusing on the ``warm start'' estimate $w$ in line \ref{line:warmstart} of \Cref{alg:pca}. First, $\|w-v\|_{2,k} \leq \|w-v\|_2 = \Theta(\|ww^\top - vv^\top\|_\fr)=O(\eps\sqrt{\log(1/\eps)}/\rho)$ (where first step is \Cref{fact:fr-to-eucl}, and the second step is because $ww^\top - vv^\top$ is rank-2 matrix). Then, since $v$ is $k$-sparse, we can assume without loss of generality by \Cref{fact:prunning} that $w$ is also $k$-sparse. Using \Cref{claim:conditional2}, this means that $\mu_{G_I}$ (the mean that we are trying to robustly estimate) is $2k$-sparse, because every $\mu_{G_\alpha}$ is a scaled version of $v-(w^\top v)w$ and both $v,w$ are $k$-sparse. The mean estimator satisfies  $\|z-\mu_{G_I}\|_{2,k} = O(\eps)$, and by \Cref{claim:conditional2}, we can keep the largest $2k$-coordinates in $z$ to obtain a $z'$ with $\|z'-\mu_{G_I}\|_{2} = O(\eps)$.

\end{proof}

Finally, regarding the last part of \Cref{thm:sparse-pca}, we have that
    \begin{align*}
       \frac{  \hat{v} \vec \Sigma \hat{v} }{\|\vec \Sigma \|_\op} 
       = \frac{1+ \rho (v^\top \hat{v})^2}{1+\rho} 
       \geq \frac{1 + \rho(1-O(\eps^2/\rho^2))}{1+\rho}
       \geq 1 - O\left( \frac{\eps^2}{\rho(1+\rho))} \right) \;.
    \end{align*}

\end{document}